\documentclass{article}

\usepackage{microtype}
\usepackage{graphicx}
\usepackage{subcaption}
\usepackage{booktabs} 
\usepackage{enumitem}
\usepackage{amsfonts}
\usepackage{mathrsfs}

\usepackage{hyperref}

\usepackage[accepted]{icml2026}
\usepackage{tcolorbox}

\usepackage{amsmath}
\usepackage{amssymb}
\usepackage{mathtools}
\usepackage{amsthm}

\usepackage[capitalize,noabbrev]{cleveref}

\theoremstyle{plain}
\newtheorem{theorem}{Theorem}[section]

\newtheorem{lemma}[theorem]{Lemma}

\theoremstyle{definition}
\newtheorem{definition}[theorem]{Definition}

\newtheorem{example}[theorem]{Example}
\theoremstyle{remark}

\usepackage[textsize=tiny]{todonotes}

\usepackage{tikz}
\usetikzlibrary{trees, arrows.meta, positioning, shapes.geometric}
\usepackage{tikz}
\usetikzlibrary{positioning, arrows.meta, shapes.geometric, automata, matrix}
\usetikzlibrary{calc, intersections, positioning, arrows.meta, backgrounds, decorations.pathreplacing}
\usepackage{tikz}
\usepackage{pgfplots}

\usepackage{bm}
\usepackage{cancel}

\newcommand{\C}{\mathcal{C}}

\newcommand{\W}{\mathcal{W}}
\newcommand{\R}{\mathcal{R}}

\newcommand{\ppf}{\mathrm{PPF}}
\newcommand{\V}{{\mathcal V}}

\newcommand{\T}{\mathcal{T}}

\newcommand{\hA}{\hat{\mathcal{A}}}

\DeclareMathOperator*{\E}{\mathbb{E}}

\icmltitlerunning{Deterministic Pareto-Optimal Policy Synthesis for MORL}

\begin{document}

\twocolumn[
  \icmltitle{Deterministic Pareto-Optimal Policy Synthesis for\\ Multi-Objective Reinforcement Learning}



  \icmlsetsymbol{equal}{*}

  \begin{icmlauthorlist}
    \icmlauthor{Aniruddha Joshi}{berk}
    \icmlauthor{Niklas Lauffer}{berk}
    \icmlauthor{Sanjit Seshia}{berk}
  \end{icmlauthorlist}

  \icmlaffiliation{berk}{University of California Berkeley, USA}

  \icmlcorrespondingauthor{Aniruddha Joshi}{aniruddhajoshi@eecs.berkeley.edu}

  \icmlkeywords{Machine Learning, Reinforcement Learning, Multi-objective Reinforcement Learning, Optimization, Artificial Intelligence}

  \vskip 0.3in
]



\printAffiliationsAndNotice{}  


\begin{abstract}

Real-world decision-making often requires balancing multiple conflicting objectives, a challenge that standard Reinforcement Learning (RL) frequently addresses by aggregating rewards into a single scalar signal. 
While effective for simple tasks, this approach often fails to capture the full spectrum of optimal trade-offs, known as the Pareto frontier. 
In this paper, we introduce a novel preference-conditioned Bellman operator, motivated from the Chebyshev scalarization, designed to compute deterministic Pareto-optimal policies for Multi-Objective Markov Decision Processes (MOMDPs).
We prove that this operator satisfies an enveloping property, where the estimated value functions upper-bound the true Pareto frontier, and demonstrate that it monotonically converges to a coverage set of this frontier.
Furthermore, we also show how to extract deterministic policies from these converged Q-estimates. 
This ensures the agent can recover a policy for any given preference, capturing the entire Pareto-optimal frontier while guaranteeing each synthesized policy remains approximately Pareto-optimal.
Experimental results validate that our algorithm successfully recovers complex trade-offs, providing a solution for deterministic Pareto-optimal policy synthesis.


\end{abstract}

\section{Introduction}

Many real-world decision-making problems are inherently multi-objective: practitioners must balance competing criteria rather than optimize a single scalar notion of performance, and the relevant trade-offs vary across users and operating regimes.

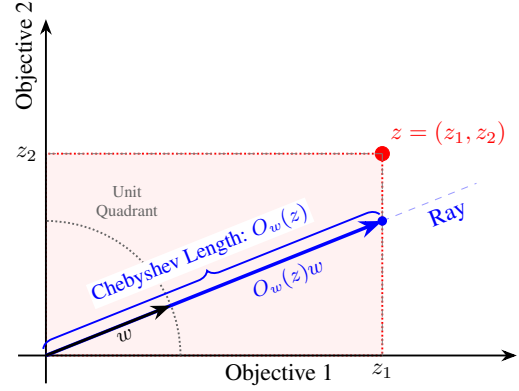
\begin{figure}
\centering
    \scalebox{0.89}{



\begin{tikzpicture}[
    scale=2, 
    >=Stealth, 
    axis/.style={->, thick, line cap=rect},
    dashed_line/.style={dashed, thin, gray},
    highlight/.style={fill=red!5, draw=red, thick, densely dotted}
]

    \coordinate (Origin) at (0,0);
    \coordinate (PointZ) at (2.5, 1.5); 
    \coordinate (PrefDir) at (1, 0.4);   
    
    \begin{scope}[on background layer]
        \draw[highlight] (Origin) rectangle (PointZ);
    \end{scope}

    \draw[axis] (-0.2,0) -- (3.5,0) node[pos=0.65, below left] {Objective 1};
    \draw[axis] (0,-0.2) -- (0,2.5) node[pos=0.9, above, rotate=90] {Objective 2};

    \filldraw[red] (PointZ) circle (1.5pt) node[anchor=south west, font=\bfseries] {$z=(z_1,z_2)$};
    \draw[dashed_line] (PointZ) -- (PointZ |- Origin) node[below, text=black] {$z_1$};
    \draw[dashed_line] (PointZ) -- (Origin |- PointZ) node[left, text=black] {$z_2$};


    \path[name path=ray] (Origin) -- ($(Origin)!3!(PrefDir)$) node[pos=0.98, below, sloped, text=blue] {Ray};;
    
    \path[name path=box_right] (PointZ |- Origin) -- (PointZ);
    \path[name path=box_top]   (Origin |- PointZ) -- (PointZ);
    
    \path[name path=unit_circle] (1,0) arc (0:90:1);

    \path [name intersections={of=ray and box_right, by=intRight}];
    \path [name intersections={of=ray and box_top, by=intTop}];
    \coordinate (ChebyshevLimit) at (intRight); 

    \path [name intersections={of=ray and unit_circle, by=UnitVectorTip}];


    \draw[densely dotted, thick, black!60] (1,0) arc (0:90:1) 
        node[pos=0.8, above right, font=\scriptsize, align=center] {Unit\\Quadrant};

    \draw[dashed, blue!40, thin] (ChebyshevLimit) -- ($(Origin)!3.2!(PrefDir)$);

    \draw[->, blue, line width=1.5pt] (Origin) -- (ChebyshevLimit);
    \fill[blue] (ChebyshevLimit) circle (1pt);

    \draw[draw=none] (Origin) -- (ChebyshevLimit) node[pos=0.7, sloped, color=blue, below, font=\small] {$O_w(z){w}$};;

    \draw[->, black, line width=1.2pt] (Origin) -- (UnitVectorTip) 
        node[pos = 0.6, sloped, below, yshift=0.1em, font=\small] {$w$};

    \draw[decorate, decoration={brace, amplitude=5pt, raise=3pt}, thick, blue] 
        (Origin) -- (ChebyshevLimit) 
        node[midway, above=10pt, sloped, align=center, font=\small, fill=white, inner sep=1pt] 
        {Chebyshev Length: $O_w(z)$};

\end{tikzpicture}

    }
    \caption{Geometric interpretation of the weighted Chebyshev operator. The preference vector $w$ (black) defines a direction in the objective space. The operator calculates the maximum scaling of $w$ that is weakly dominated by the value vector $z=(z_1,z_2)$.
    Geometric interpretation of the weighted Chebyshev operator. The unit preference vector $w$ (black) indicates the direction. The operator determines the maximum factor by which this vector can be scaled until it reaches the boundary of the region dominated by the reference point $z$. The magnitude of this scaled vector represents the Chebyshev value.}
    \label{fig:left_plot}
\end{figure}

This is evident across a wide range of domains, including circuit design (power, performance, area, reliability) \cite{5396122, VISAN2022109987,10.1145/3649329.3657318, 11092100}; distributed computing (energy, latency, throughput) \cite{da2019multi,qin2020energy}; drug and molecule design (efficacy, toxicity, synthesizability) \cite{zhou2019optimization}; recommender systems (accuracy, diversity, fairness) \cite{lacerda2017multi}; robotics (speed, safety, energy) \cite{huang2019learning}; satellite communications (throughput, delay, power) \cite{hu2020dynamic}; control (tracking, robustness, effort) \cite{10.1007/978-3-031-83191-1_17}; and mechanical design (weight, strength, manufacturability) \cite{10.1115/1.4069046}. 
Reinforcement learning (RL) is increasingly used as a general optimization, control, and exploration mechanism in such settings, but standard RL typically assumes a single reward signal \cite{sutton2018reinforcement}. When multiple objectives are present, practitioners commonly apply scalarisation techniques (e.g., hand-tuned weighted sums) to collapse them into a single objective; while convenient, this can eliminate important trade-offs and may yield policies that perform well on average but fail to satisfy critical preferences in specific scenarios. These limitations motivate a truly multi-objective approach: multi-objective reinforcement learning (MORL), which seeks to learn the full set of Pareto-optimal solutions.

In this work, we introduce a novel Bellman operator parameterized by preference weights. This formulation allows us to obtain a deterministic Pareto-optimal policy corresponding to any desired trade-off. 
We provide a proof of the operator's convergence, while guaranteeing that the extracted policies approximately cover the entire Pareto frontier and are individually approximately Pareto-optimal.
Further we empirically demonstrate that, our procedure can retrieve the complete Pareto frontier.


The \textbf{main contributions} of this work are as follows:
\begin{enumerate}[topsep=2pt,itemsep=2pt,parsep=0pt]
    \item We introduce a model-based Bellman operator parameterized by preferences, leveraging Chebyshev scalarization. We derive its error bounds and prove that it converges asymptotically to the Pareto-optimal values for deterministic, non-stationary policies in MOMDPs.
    \item From these converged Q-estimates, we demonstrate how to extract deterministic policies requiring only a single-step transition memory. We prove that this approach yields a policy set that covers the Pareto frontier while guaranteeing that each recovered policy is approximately non-dominated within the space of all non-stationary deterministic policies.
    \item We provide empirical evidence demonstrating that our algorithm successfully converges to the Pareto frontier, capturing all trade-offs and recovers a set of deterministic Pareto-optimal policies for all of them.
\end{enumerate}

The remainder of this paper is organized as follows. Section~\ref{rel_work} reviews related literature. 
Section~\ref{preliminaries} introduces the Multi-Objective Markov Decision Process framework and defines the problem of synthesizing a Pareto coverage set. 
Section~\ref{sec:chebyshev_scalarization} introduces Chebyshev scalarization as our core optimization mechanism, provides geometric intuition for this operator, and proves that a norm-maximization tie-breaking rule yields a necessary and sufficient condition for exactly characterizing the Pareto frontier. 
Section~\ref{sec:bellman_update} details our novel, preference-aware Bellman update, while Section~\ref{sec:policy} outlines the resulting deterministic policy synthesis algorithm. Section~\ref{sec:experiments} validates the convergence and coverage capabilities of our approach through empirical experiments.
Finally, Section~\ref{sec:conclusion} concludes the paper and outlines directions for future work.

\section{Related Work}\label{rel_work}

We contextualize our contributions within three primary paradigms of Multi-Objective Reinforcement Learning: \textit{linear scalarization}, \textit{preference-based learning}, and \textit{explicit set maintenance}. While comprehensive surveys provide a broader view~\cite{10.5555/2591248.2591251, survey1, ZHANG2023526}, we focus on these specific categories to highlight a fundamental challenge: the difficulty of efficiently obtaining deterministic policies that cover the full Pareto frontier while maintaining rigorous convergence guarantees.

\textbf{Linear Scalarization.} A common approach to MORL involves reducing the vector-valued reward to a scalar via a weighted sum. This technique has been applied across diverse domains, including manufacturing~\cite{inproceedings_aiss}, hydroelectric control~\cite{Shabani_2009}, energy management~\cite{10.4018/jats.2009040104}, and grid computing~\cite{10.1145/1555301.1555311}, alongside the development of specialized algorithms for this setting~\cite{10.1145/1390156.1390162, 10.5555/3454287.3455598, inproceedings_Castelletti}. However, linear scalarization is fundamentally restricted to recovering solutions on the convex coverage set of the Pareto frontier, failing to identify optimal trade-offs in non-convex regions common in complex real-world tasks~\cite{10.5555/2591248.2591251, 6889637, 10.1007/978-3-030-22649-7_25}.

\textbf{Preference-Based Approaches.} These methods generalize linear scalarization by parameterizing the value function with a vector of weights. The weighted Chebyshev scalarization, originally characterized by Bowman~\cite{10.1007/978-3-642-87563-2_5}, has demonstrated empirical success in multi-objective evolutionary algorithms~\cite{4631208}. In the MORL context, \cite{inproceedings_van} utilize this metric to learn multiple policies, performing weight sweeps (fixing specific weights) to approximate the Pareto curve. Similarly, Reymond et al.~\cite{Reymond2019ParetoDQNAT} extend these concepts using neural network approximations of Pareto-optimal Q-values.

Our work builds upon this foundation by defining a preference space and estimating Q-values parameterized by these preferences. 
However, unlike methods that fix scalarization weights to learn separate policies, we integrate the preference vector directly into the learning process. 
This Bellman operator allows the policy to dynamically optimize against varying preferences recursively, resulting in a single succinct representation parameterized by the weights.
Furthermore, we provide a proof of asymptotic convergence to deterministic Pareto-optimal policies.

Modern approaches~\cite{Yamamoto2019HypervolumeBasedMR,van2014multi,ruiz2017temporal,mandow2018pruning} focus on maintaining explicit sets of Pareto-optimal estimates. However, extracting a policy to achieve a specific Pareto-optimal value from these sets presents a fundamental challenge: \textit{alignment consistency}. In these settings, the value of a state is represented by a set of vectors, each corresponding to a distinct optimal trade-off. To execute a policy, an agent must effectively select a target trade-off vector from the current state's set. Crucially, this choice imposes a strict constraint on the subsequent step: to realize the chosen trade-off, the agent must identify and select the specific corresponding vector from the set of estimates at the next state that contributed to the current value. In contrast to scalar RL, where a greedy maximization over a single value suffices, the MORL agent must maintain this specific alignment across the trajectory. If the agent fails to align the next-state selection with the previous state's trade-off target, the global policy becomes suboptimal. As highlighted in \cite{van2014multi} and \cite{article_Ruiz}, arbitrarily selecting any trade-off without ensuring this alignment breaks the optimality chain. Consequently, these methods rely on heuristic action selection mechanisms that lack formal guarantees, and thus the executed policy is not guaranteed to realize the target Pareto-optimal value.

Our approach addresses these limitations simultaneously. First, by parameterizing the value function with preferences, we eliminate the need for an explicit set representation, making our method amenable to standard function approximation. Second, we establish asymptotic convergence guarantees for the estimates similar to those of White~\cite{white1982multi}; however, we obtain policies without the computationally intractable requirement of storing action histories. 
Instead, we resolve the alignment consistency problem efficiently: our procedure dynamically determines the preference for the subsequent state based on the selected action and current preference.
This mechanism allows us to recover non-stationary, deterministic Pareto-optimal policies by simply changing the local preference parameter over the rollout, and we further prove convergence of our policy to the frontier.

\textbf{Explicit Set-Based Approaches.} Works by White~\cite{white1982multi} established that generating the complete Pareto frontier using deterministic policies requires non-stationary behavior, as stationary policies may be dominated in multi-objective settings. White proposed a Bellman update operating on sets of estimates and established value convergence (a proof later revised by Mifrani~\cite{DBLP:journals/anor/Mifrani25}). However, retrieving the corresponding policy requires storing the entire history of actions across all updates ($n$ steps) and may scale exponentially. As noted by \cite{4220828}, this memory requirement is computationally infeasible. To circumvent this, algorithms like CON-MODP~\cite{4220828} restrict the search space to stationary policies, thereby sacrificing the ability to recover the Pareto frontier.

\section{Preliminaries and Problem Formulation}\label{preliminaries}

In this section, we formalize the problem of synthesizing a complete and parsimonious set of Pareto-optimal policies. We first introduce the necessary preliminaries and definitions.

\begin{definition}[Multi-Objective Markov Decision Process (MOMDP)]\label{momdp}
    A Multi-Objective Markov Decision Process (MOMDP) is a tuple $\mathcal{M} := (S, A, T, \gamma, s_0, R)$, where $S$ is a set of states, $A$ is a set of actions, $T: S \times A \times S \rightarrow [0,1]$ is a probabilistic transition function, $\gamma \in [0,1)$ is a discount factor, $s_0\in S$ is the initial state, and $R: S \times A \rightarrow \mathbb{R}^d_{\geq 0}$ is a $\bm d$-\textbf{dimensional vector-valued reward}.
\end{definition}

\begin{example} 
    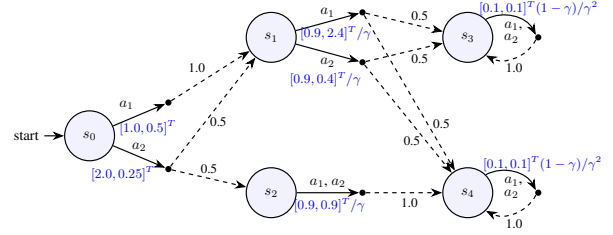
\begin{figure}
        \centering
        \scalebox{0.55}{
        \begin{tikzpicture}[
    node distance=2.5cm and 3cm,
    state/.style={circle, draw, minimum size=1.2cm, thick, fill=blue!5},
    initial text={start},
    every initial by arrow/.style={thick, -{Stealth[scale=1.2]}},
    action/.style={circle, fill=black, inner sep=1.5pt},
    reward/.style={font=\small, color=blue!70!black},
    prob/.style={font=\footnotesize, pos=0.7},
    edge/.style={thick, -{Stealth[scale=1.2]}},
    stoch_edge/.style={thick, dashed, -{Stealth[scale=1.2]}}
]

    \node[state, initial] (S0) {$s_0$};
    
    \node[state, above right=1.5cm and 3.5cm of S0] (S1) {$s_1$};
    \node[state, below right=0.5cm and 3.5cm of S0] (S2) {$s_2$};
    
    \node[state, right=3.5cm of S1] (S3) {$s_3$};
    \node[state, right=3.5cm of S2] (S4) {$s_4$};

    \node[action, right=1.2cm of S0, yshift=0.8cm] (actA) {};
    \draw[edge] (S0) -- node[above left] {$a_1$} (actA);
    
    \draw[stoch_edge] (actA) -- node[prob, left, pos=0.5, yshift=0.2cm] {1.0} (S1);
        
    \node[coordinate, label={[text=blue, xshift=0.2cm, yshift=-0.3cm]below:{\small $[1.0, 0.5]^T$}}] (actA) at ([shift={(1.2cm,0.8cm)}]S0) {};

    \node[action, right=1.2cm of S0, yshift=-0.8cm] (actB) {};

    \draw[edge] (S0) -- node[above ] {$a_2$} (actB);
    
    \node[coordinate, label={[text=blue, xshift=-0.4cm, yshift=0.2cm]below:{\small $[2.0, 0.25]^T$}}] (rewB) at ([shift={(1.2cm,-0.8cm)}]S0) {};

    \draw[stoch_edge] (actB) -- node[prob, above, pos=0.5, yshift=-0.0cm] {0.5} (S2);

    \draw[stoch_edge] (actB) -- node[prob, below right, pos=0.5, yshift=0.0cm, xshift=-0.2cm] {0.5} (S1);

    \node[action, right=1.5cm of S2] (actStoch) {};
    \draw[edge] (S2) -- node[above] {$a_1$, $a_2$} (actStoch);

    \node[coordinate, label={[text=blue, xshift=0.2cm, yshift=0.8cm]below:{\small $[0.9, 0.9]^T/\gamma$}}] (rewS2) at ([shift={(1.2cm,-0.8cm)}]S2) {};

    \draw[stoch_edge] (actStoch) -- node[prob, below right, pos=0.5, yshift=0.0cm, xshift=-0.2cm] {1.0} (S4);
    
    

    \node[action, right=1.5cm of S1, yshift=0.6cm] (actC) {};
    \draw[edge] (S1) -- node[above] {$a_1$} (actC);
    
    \node[coordinate, label={[text=blue, xshift=0.35cm, yshift=0.4cm]below:{\small $[0.9, 2.4]^T/\gamma $}}] (rewS1a) at ([shift={(1.2cm,0.0cm)}]S1) {};

    \draw[stoch_edge] (actC) -- node[prob, above] {0.5} (S3);

    \draw[stoch_edge] (actC) -- node[prob, right] {0.5} (S4);
    
    \node[action, right=1.5cm of S1, yshift=-0.6cm] (actD) {};
    \draw[edge] (S1) -- node[below] {$a_2$} (actD);

    \node[coordinate, label={[text=blue, xshift=0.1cm, yshift=-0.4cm]below:{\small $[0.9, 0.4]^T/\gamma$}}] (rewS1a2) at ([shift={(1.2cm,-0.3cm)}]S1) {};

    \draw[stoch_edge] (actD) --  node[prob, below] {0.5} (S3);

    \draw[stoch_edge] (actD) -- node[prob, left] {0.5} (S4);


    \node[action, right=1cm of S3] (actL3) {};

    \node[coordinate, label={[text=blue, xshift=0.5cm, yshift=0.4cm]above:{\small $[0.1, 0.1]^T(1-\gamma)/\gamma^2$}}] (rewS13a) at ([shift={(1.2cm,0.0cm)}]S3) {};

    \node[coordinate, label={[xshift=-0.1cm, yshift=0.5cm]below:{$a_1$,}}] (rewS13a) at ([shift={(1.2cm,0.0cm)}]S3) {};

    \node[coordinate, label={[xshift=-0.15cm, yshift=0.2cm]below:{$a_2$}}] (rewS134a) at ([shift={(1.2cm,0.0cm)}]S3) {};

    \draw[edge] (S3) to[bend left=45] (actL3);
    
    \draw[stoch_edge] (actL3) to[bend left=45] 
    node[prob, below, pos=0.5] {1.0} (S3);


        \node[action, right=1cm of S4] (actL34) {};

    \node[coordinate, label={[text=blue, xshift=0.6cm, yshift=0.4cm]above:{\small $[0.1, 0.1]^T(1-\gamma)/\gamma^2$}}] (rewS134a) at ([shift={(1.2cm,0.0cm)}]S4) {};

    \node[coordinate, label={[xshift=-0.1cm, yshift=0.5cm]below:{$a_1$,}}] (rewS134a) at ([shift={(1.2cm,0.0cm)}]S4) {};

    \node[coordinate, label={[xshift=-0.15cm, yshift=0.2cm]below:{$a_2$}}] (rewS1345a) at ([shift={(1.2cm,0.0cm)}]S4) {};

    \draw[edge] (S4) to[bend left=45] (actL34);
    
    \draw[stoch_edge] (actL34) to[bend left=45] 
    node[prob, below, pos=0.5] {1.0} (S4);

    

\end{tikzpicture}}
        \caption{Solid arrows represent actions and are annotated with action labels and rewards, while dashed arrows represent stochastic transitions annotated with their probabilities.}
        \label{fig:MOMDP}
    \end{figure}
    Consider the MOMDP with 2-dimensional rewards illustrated in Figure~\ref{fig:MOMDP}. As an example, executing action $a_1$ from the initial state $s_0$ produces the reward $\bigl(\begin{smallmatrix} 1.0 \\ 0.5 \end{smallmatrix}\bigr)\in \mathbb{R}^2_{\ge 0}$ and transitions to state $s_1$.
\end{example}

We focus on \textbf{deterministic policies} $\pi$ that map a state-action history $h_t$ and the current state $s \in S$ to an action $a \in A$. We denote the action selected at state $s$ given history $h_t$ by $\pi(s \mid h_t)$, and the induced future policy by $\pi(\cdot \mid h_t)$. 

Let $\Delta(S)$ denote the set of probability distributions over $S$. For an initial state distribution $\mathcal{I} \in \Delta(S)$, the \textbf{Value Function} (or \emph{value vector}) $\mathbf{V}^\pi(\mathcal{I}) \in \mathbb{R}^d_{\geq 0}$ is the expected discounted return obtained by executing policy $\pi$:
\[
\mathbf{V}^\pi(\mathcal{I}) = \mathbb{E}_{s_0 \sim \mathcal{I}, \pi} \left[ \sum_{t=0}^{\infty} \gamma^t R(s_t, a_t) \right]
\]

    \begin{figure}
        \scalebox{0.9}{\begin{tikzpicture}[scale=1.5]
    \draw[very thin, gray!20] (0,0) grid (4,4);
    

    \draw[thick, ->] (-0.2,0) -- (4.2,0)
    node[midway, below, sloped, inner sep=0.5cm] {Objective 1};
    \draw[thick, ->] (0,-0.2) -- (0,4.2)
    node[midway, above, sloped, inner sep=0.5cm] {Objective 2};
    \foreach \x in {1,2,3,4} \draw (\x,1pt) -- (\x,-1pt) node[below, font=\small] {\x};
    \foreach \y in {1,2,3,4} \draw (1pt,\y) -- (-1pt,\y) node[left, font=\small] {\y};

    
\matrix [
    draw,
    fill=white,
    matrix of nodes,
    anchor=north east,
    xshift=0.8cm,       
    yshift=-0.1cm,
    row sep=-1pt,        
    column sep=2pt,      
    inner sep=3pt,       
    nodes={anchor=center, inner sep=2pt, font=\scriptsize}, 
    scale=0.8,
    transform shape
] at (current bounding box.north east) {
    \node[circle, fill=red!60, minimum size=5pt] {};  & \node[anchor=west] {PO value}; \\
    \node[regular polygon, regular polygon sides=3, fill=red!60, minimum size=6pt, inner sep=0pt] {}; & \node[anchor=west] {Weak PO value}; \\
    \node[circle, fill=blue!60, minimum size=5pt] {}; & \node[anchor=west] {Dominated value}; \\
    \draw[black, dashed, thick] (-0.2,0) -- (0.26,0); & \node[anchor=west] {Pareto performance}; \\
    \draw[draw=none] ;& \node[anchor=west] { frontier (PPF)};\\
};

    \filldraw[blue!60] (2,1) circle (2pt);
        
    \node[regular polygon, regular polygon sides=3, fill=red!60, inner sep=1pt, minimum size=10pt] at (3,1) {};

    \filldraw[red!60] (3,2) circle (2pt);
    \filldraw[red!60] (2,3) circle (2pt);

    \draw[draw=none] (0.7, 0.6) node[] {\scriptsize $(1,0.5)$};

    \draw[draw=none] (2.1,0.1) node[] {\scriptsize $(2.0,0.25)$};

    \draw[draw=none] (3.45,0.75) node[] {\scriptsize $(3.0,0.75)$};

    \draw[draw=none] (3.45,1.25) node[] {\scriptsize $(3.0,1.25)$};

    \draw[draw=none] (3.45,2.75) node[] {\scriptsize $(3.0,2.75)$};

    \draw[blue, thick, -Stealth] (0,0) -- (1,0.5) node[pos=0.5, above, sloped, black, inner sep=0.5pt] {$a_1$};;
    \draw[blue, thick, -Stealth] (1,0.5) -- (2,3) node[pos=0.5, above, sloped, black, inner sep=0.5pt] {$a_1$};;

   \draw[blue, thick, -Stealth] (0,0) -- (1,0.5);
   \draw[blue, thick, -Stealth] (1,0.5) -- (2,1) node[pos=0.5, above, sloped, black, inner sep=0.5pt] {$a_2$};;

    \draw[blue, thick, -Stealth] (0,0) -- (2.0,0.25) node[pos=0.5, above, sloped, black, inner sep=0.5pt] {$a_2$};;
    
    \draw[blue, dashed, -Stealth] (2.0,0.25) -- (3.0,2.75) node[midway, above, sloped, black, inner sep=0.5pt] {$a_1$};
    

    \draw[blue, dashed, -Stealth] (2.0,0.25) -- (3.0,1.25) node[pos=0.55, above, sloped, black, inner sep=0.5pt] {};

    \draw[blue, thick, -Stealth] (2.0,0.25) -- (3.0,2.0);

    \draw[blue, dashed, -Stealth] (2.0,0.25) -- (3.0,0.75) node[midway, below, sloped, black, inner sep=0.5pt] {$a_2$};
    \draw[blue, dashed, -Stealth] (2.0,0.25) -- (3.0,1.25) node[pos=0.7, above, sloped, black, inner sep=0.5pt] {$a_1$, $a_2$};
    \draw[blue, thick, -Stealth] (2.0,0.25) -- (3.0,1);

    \filldraw[blue] (3.05,2) circle (0pt) node[right, blue] {$\pi_3 (3,2)$};
    \filldraw[blue] (2,3) circle (0pt) node[above right, blue] {$\pi_1 (2,3)$};
    \filldraw[blue] (1.9,1) circle (0pt) node[above, blue] {$\pi_2 (2,1)$};
    \filldraw[blue] (3.05,1) circle (0pt) node[ right, blue] {$\pi_4 (3,1)$};

    \draw[dashed, thick, black!60] (0,3) -- (2,3) -- (2,2) -- (3,2) -- (3,0);


\end{tikzpicture}}
        \caption{Multi-objective value vectors. Arrows show rewards accrued from deterministic (solid) and stochastic (dashed, $p=0.5$) transitions from $s_0$, $s_1$ and $s_2$. Pareto-optimal and weakly Pareto-optimal vectors are marked in red (circles and triangle, respectively), while dominated values are in blue. The dashed black line indicates the Pareto performance frontier.}
        \label{fig:four_value_vectors}
    \end{figure}
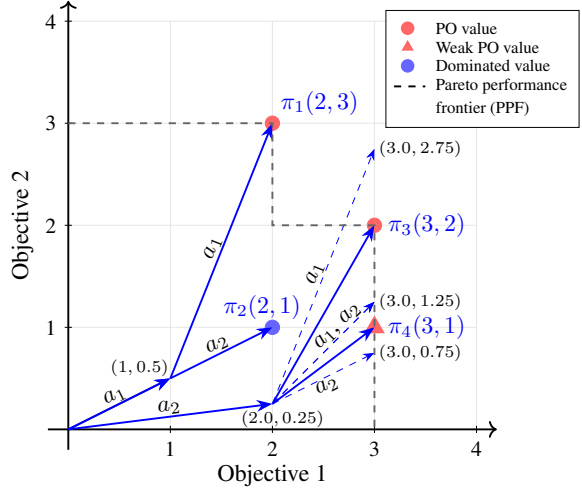

\begin{example}


In the MOMDP (Figure~\ref{fig:MOMDP}), action choices affect transitions and rewards only at states $s_0$ and $s_1$; at states $s_2$, $s_3$, and $s_4$, the outcomes are independent of the action chosen. This results in four distinct deterministic policies, corresponding to the action pairs $\pi_1=(a_1, a_1)$, $\pi_2=(a_1, a_2)$, $\pi_3=(a_2, a_1)$, and $\pi_4=(a_2, a_2)$, where the first and second components denote the actions taken at states $s_0$ and $s_1$ regardless of the history. These policies yield the value vectors $(2,3)$, $(2,1)$, $(3,2)$, and $(3,1)$, respectively (see Figure~\ref{fig:four_value_vectors}).




\end{example}

\textbf{Notation.} Let $\Pi$ denote the set of deterministic policies. For vectors $u, v \in \mathbb{R}^d_{\geq 0}$, $u$ \emph{weakly dominates} $v$ ($v \preceq u$) if $u_i \geq v_i$ for all $i$, and \emph{strictly dominates} ($v \prec u$) if $v \preceq u$ and $u \neq v$. We use the shorthand $V^\pi(s)$ to denote the value of policy $\pi$ starting from state $s$. Let $\overline{\mathbf{1}}$ denote a $d$-dimensional vector of all ones. Assuming a maximum component-wise reward $R_{\max}\in \mathbb{R}_{\geq 0}$, all value vectors are bounded by $\mathcal{R} := \frac{R_{\max}}{1-\gamma}\overline{\mathbf{1}}$.

We note that two vectors $u, v \in \mathbb{R}^d_{\geq 0}$ are \emph{incomparable} if neither $u \preceq v$ nor $v \preceq u$. This occurs if and only if there exist indices $i$ and $j$ such that $u_i < v_i$ and $u_j > v_j$.

A policy $\pi \in \Pi$ is \textbf{Pareto-optimal (PO)} if it is not strictly dominated by any other policy; formally, there exists no $\pi' \in \Pi$ such that $V^\pi(s_0) \prec V^{\pi'}(s_0)$.\footnote{While dominance is defined here with respect to the initial state $s_0$, the concept naturally extends to any initial state distribution $\mathcal{I} \in \Delta(S)$ by substituting $V^\pi(s_0)$ with ${V}^\pi(\mathcal{I})$.} The set of such policies, denoted $\Pi^{PO}$, represents the trade-offs where no objective can be improved without sacrificing another.




\begin{example}
Figure~\ref{fig:four_value_vectors} illustrates the value vectors for the four policies: $\pi_1$ and $\pi_3$ are Pareto-optimal, whereas $\pi_2$ is dominated by $\pi_1$ ($V^{\pi_2}(s_0) \preceq V^{\pi_1}(s_0)$) and $\pi_4$ is dominated by $\pi_3$ ($V^{\pi_4}(s_0) \preceq V^{\pi_3}(s_0)$). Furthermore, policies $\pi_1$ and $\pi_3$ are incomparable, as neither of their value vectors dominates the other.
\end{example}

Since multiple policies may yield identical value vectors, $\Pi^{PO}$ can contain redundant policies. We therefore seek a \textbf{Pareto Coverage Set} $\mathcal{C}_{PO} \subseteq \Pi^{PO}$ that serves as a representative subset capable of characterizing the entire frontier of optimal values without the redundancy of the full policy space; formally, $\forall \pi \in \Pi^{PO}, \exists \pi' \in \mathcal{C}_{PO}$ such that $V^{\pi'}(s_0) = V^{\pi}(s_0)$.

Identifying such a set $\mathcal{C}_{PO}$ is non-trivial and motivates our problem definition to synthesize a parsimonious policy set that satisfies coverage:

\begin{tcolorbox}[
    colback=white,      
    colframe=black,     
    colbacktitle=white, 
    coltitle=black,     
    title=Problem Statement\\
    Complete And Parsimonious Synthesis (CAPS),
    fonttitle=\bfseries, 
    center title,       
    titlerule=0pt       
]
Given a Multi-Objective Markov Decision Process $M=(S,A,T,\gamma,s_0,R)$,  synthesize a Pareto coverage set $\mathcal{C}_{PO} \subseteq \Pi^{PO}$ that satisfies the following conditions:

\begin{center} 
 \begin{itemize}
     \item \textbf{Coverage}: Every Pareto-optimal policy is represented by a policy in the set $\C_{PO}$: 

     \hspace{2em}
     $\forall \pi \in \Pi^{PO}, \exists \pi' \in \mathcal{C}_{PO} \text{ such that } $
     
     \hspace{4em} $V^{\pi}(s_0) = V^{\pi'}(s_0) $
     \item \textbf{Parsimony}: Every policy in the coverage set is itself Pareto-optimal: 
     
    \hspace{2em}
    $\forall \pi' \in \C_{PO}, \exists \pi \in \Pi^{PO} \text{ such that } $
     
    \hspace{4em} $V^{\pi'}(s_0) = V^{\pi}(s_0) $
     
 \end{itemize}
\end{center}
\end{tcolorbox}

Since value functions are defined as the limits of infinite sums, evaluating them exactly may not be practically possible without truncating after a finite horizon. 
To account for the inherent error of this truncation, we relax the exact equality constraints and define the $\epsilon$-CAPS problem as follows:

\begin{tcolorbox}[
    colback=white,      
    colframe=black,     
    colbacktitle=white, 
    coltitle=black,     
    title=Problem Statement\\
    Approximate Complete And Parsimonious Synthesis ($\epsilon$-CAPS),
    fonttitle=\bfseries, 
    center title,       
    titlerule=0pt       
]
Given a Multi-Objective Markov Decision Process $M=(S,A,T,\gamma,s_0,R)$ and an $\epsilon \in \mathbb{R}_{> 0}$, synthesize an approximate Pareto coverage set $\mathcal{C}^{\epsilon}_{PO}$ that satisfies the following conditions:

\begin{center} 
 \begin{itemize}
     \item \textbf{Approximate Coverage}: Every Pareto-optimal policy is approximately dominated by a policy in the approximate coverage set $\C^\epsilon_{PO}$: 
     
     \hspace{2em}
     $\forall \pi \in \Pi^{PO}, \exists \pi' \in \mathcal{C}^\epsilon_{PO} \text{ such that } $
     
     \hspace{4em} $V^{\pi}(s_0) \preceq V^{\pi'}(s_0)  + \epsilon \overline{\mathbf{1}} $
     \item \textbf{Approximate Parsimony}: Every policy in the approximate coverage set is itself approximately Pareto-optimal: 
     
    \hspace{2em}
    $\forall \pi' \in \C^\epsilon_{PO}, \exists \pi \in \Pi^{PO} \text{ such that } $
     
    \hspace{4em} $V^{\pi'}(s_0) \preceq V^{\pi}(s_0) + \epsilon \overline{\mathbf{1}} $
     
 \end{itemize}
\end{center}
\end{tcolorbox}

Given a small margin $\epsilon > 0$, the approximate coverage property ensures that the value of every Pareto-optimal policy is weakly dominated by the value of a policy in $\mathcal{C}_{PO}^\epsilon$, up to an additive factor of $\epsilon$ across all components. Similarly, the approximate parsimony property guarantees that every policy in $\mathcal{C}_{PO}^\epsilon$ is weakly dominated by an actual Pareto-optimal policy, up to the same $\epsilon$ margin.

To solve the $\epsilon$-CAPS problem, we next introduce an additional parameter called a \emph{preference}. Optimizing for a specific preference yields an approximately Pareto-optimal policy; consequently, systematically varying this preference allows us to construct the full approximate Pareto coverage set $\mathcal{C}_{PO}^\epsilon$. 

Our approach proceeds in two main steps. First, assuming a known probabilistic model of the system, we introduce a Bellman update procedure and prove that the resulting Q-value estimates converge arbitrarily close to the true Pareto-optimal values. Second, we demonstrate how to extract a history-independent policy $\pi$ from these Q-estimates. This independence relies on a core property we establish regarding \emph{greedy Pareto-optimal policies}: there exist deterministic, Pareto-optimal policies where, upon taking action $a$ in state $s$ and transitioning to $s'$, the subsequent behavior perfectly matches another deterministic Pareto-optimal policy originating from $s'$. Consequently, instead of storing the full execution history $h_t$, we can greedily select actions by dynamically updating the preference parameter based solely on the most recent state transition.


\section{Preference-Based Chebyshev Scalarization}\label{sec:chebyshev_scalarization}


A common approach to solving the Pareto-optimal policy synthesis problem in MOMDPs is via scalarization. 
In this framework, a  non-negative, unit-normalized preference vector encodes the relative importance of the objectives, and a scalarization function maps vector-valued returns to scalar values. For each preference vector, this yields a distinct single-objective optimization problem, whose solution corresponds to a policy tailored to that preference. Solving these problems over a collection of preference vectors produces a set of policies that characterize the Pareto frontier. Our approach builds on this idea; we next formally define the preference space and the Chebyshev scalarization function.



\begin{definition}[Preference Space]
We define the preference space $\mathcal{W}$ as the set of all non-negative, unit-normalized weight vectors:
\[
\mathcal{W} := \Big\{ w \in \mathbb{R}^d_{\geq 0} \;\Big|\; \|w\|_2 = 1,\; w_i \ge 0,\; \forall i = 1,\dots,d \Big\}.
\]
We refer to any vector $w \in \mathcal{W}$ as a \textit{preference}. Conceptually, each component $w_i$ acts as a weight assigned to the $i$-th objective, representing its relative importance.
\end{definition}

\begin{definition}[Chebyshev Scalarization]\label{def:cheb_scalarization}
Let $\mathcal{W}$ be the preference space. The \textit{Weighted Chebyshev scalarization} is a function $O: \mathcal{W} \times \mathbb{R}^d_{\geq 0} \rightarrow \mathbb{R}_{\geq 0}$ that maps a preference vector $w\in \W$ and a value vector $V\in  \mathbb{R}^d_{\geq 0}$ to a non-negative scalar score in $\mathbb{R}_{\geq 0}$. It is defined as:
$$    O(w, V) := \min_{i : w_i > 0} \left( \frac{V_i}{w_i} \right)
$$
\end{definition}

For a fixed preference vector \(w \in \mathcal W\), we write $O_w(V) := O(w, V)$
to denote the Chebyshev scalarization induced by the preference \(w\).

\textbf{Geometric Intuition.} Geometrically, the weighted Chebyshev operator $O(w,V)$ identifies the largest scale factor $\alpha \geq 0$ such that the vector $\alpha w$ remains weakly dominated by $V$. 
It measures how far a ray originating from the origin in the direction of $w$ can extend before it is constrained by the components of $V$ (illustrated in Figure~\ref{fig:left_plot}). This relationship is formalized by the optimization: $$ \max \alpha \quad \text{ subject to } \quad \alpha w \preceq V $$
The constraint $\alpha w \preceq V$ requires that $\alpha w_i \leq V_i$ for all $i$. Consequently, the optimization simplifies to $\alpha \leq \min\limits_{i : w_i > 0} \left( \frac{V_i}{w_i} \right)$. Thus, the resulting scalar value is determined strictly by the $i$-th \emph{bottleneck component} that produces this minimum ratio. This component represents the first coordinate of $V$ that restricts the scaling along the ray.

We note that for a fixed preference $w$, multiple distinct value vectors (including those that are strictly dominated) may yield the exact same scalar score. Because the operator is sensitive only to the limiting bottleneck component, increasing any non-bottleneck objective $V_j$ does not change the resulting value. This phenomenon, where an improvement in one objective is not reflected in the scalarization, is illustrated in Example~\ref{ex:scalar_blindness} below.

\begin{example} \label{ex:scalar_blindness} Consider four policies $\pi_1,\ldots\pi_4$ with value vectors $V^1=(2,3)$, $V^2=(2,1)$, $V^3=(3,2)$, and $V^4=(3,1)$, as shown in Figure~\ref{fig:four_value_vectors}. 
Suppose we have a preference vector $w=\frac{(3,0.5)}{\|(3,0.5)\|_2}$. 
For both $V^3$ and $V^4$, the first objective is the bottleneck since: $\frac{V^3_1}{w_1}=\frac{V^4_1}{w_1}=\|(3,0.5)\|_2$. While $V^3$ strictly dominates $V^2$ (since $2>1$ in the second objective), both yield the same scalarized value: 
$O_w(V^3)=O_w(V^4)=\|(3,0.5)\|_2$. 
Consequently, an optimizer using this preference $w$ would be indifferent between the two. However, for a different preference $w'=\frac{(2,1)}{\|(2,1)\|}$, the bottleneck shifts, and only $V^3$ would remain as the maximum.

\end{example}



While Example~\ref{ex:scalar_blindness} highlights the operator's indifference between specific value vectors, it also suggests that for any preference direction $w$, there exists a  limit to how far the ray $\alpha w$ can extend within the space of achievable outcomes. 
We formalize this boundary in Definition~\ref{def:ppf}.



\begin{definition}[Pareto Performance Frontier (PPF)]\label{def:ppf}
    Let $\mathcal{V} \subset \mathbb{R}_{\geq 0}^d$ be the set of all achievable value vectors, such that $\mathcal{V} \neq \{\bm 0\}$.
    Let $\mathcal{V}^{PO} \subseteq \mathcal{V}$ denote the subset of Pareto-optimal points in $\mathcal{V}$.
    The \textbf{Pareto Performance Frontier} $\ppf \subset \mathbb{R}_{\geq 0}^d$ is defined as the set of value vectors characterizing the upper boundary of the achievable region. A point $v$ belongs to the $\ppf$ if it is weakly dominated by a Pareto-optimal point, yet cannot be strictly improved in all objectives simultaneously by any achievable vector. Formally:
    \begin{align*}
        \ppf = \Big\{ & v \in \mathbb{R}_{\geq 0}^d \;\Big|\; \exists v^* \in \mathcal{V}^{PO} \text{ s.t. } v \preceq v^* \\
        & \text{and } \nexists v'\in \mathcal{V} \text{ s.t. } \forall i=1,\ldots, d, \; v_i < v'_i \Big\}
    \end{align*}
\end{definition}

The frontier separates the achievable performance region from the unachievable region. The achievable region is defined as the set of points weakly dominated by the Pareto set; conversely, the unachievable region comprises values that are strictly unattainable by any policy.


Figure~\ref{fig:ppf} visualizes the $\ppf$ as a dashed boundary separating the achievable region (weakly dominated by $\mathcal{V}^{PO}$) from the strictly unachievable region. By definition, no point on the $\ppf$ allows simultaneous improvement in all objectives.

\begin{example}
     Figure~\ref{fig:ppf} illustrates $\ppf$ with a dashed boundary. This frontier demarcates the achievable region, where all feasible points lie on or below the boundary, from the unachievable region located above it. It is easy to see that no point on the $\ppf$ can be improved in all objectives simultaneously.
     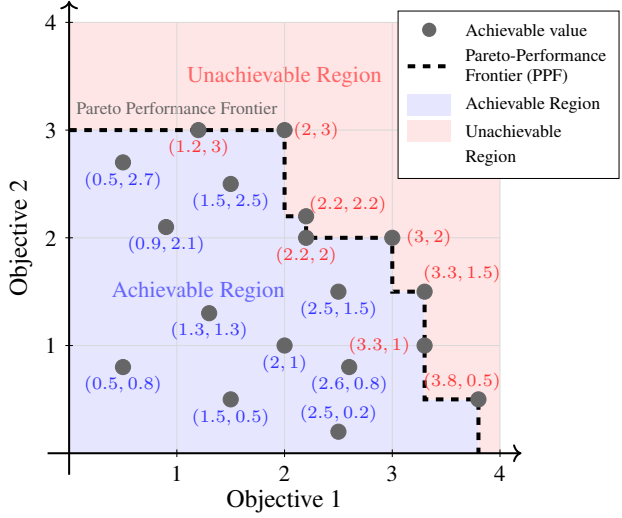
\begin{figure}
         \hspace{-1em}
         \scalebox{0.95}{
         \begin{tikzpicture}[scale=1.5]

    \def\ppfpath{ (0,3) -- (2,3) -- (2,2.2) -- (2.2,2.2) -- (2.2, 2) -- (3,2) -- (3,1.5) -- (3.3,1.5) -- (3.3,0.5) -- (3.8,0.5) -- (3.8,0)}

    \fill[blue!10] (0,0) -- (0,3) -- (2,3) -- (2,2.2) -- (2.2,2.2) -- (2.2,2) -- (3,2) -- (3,1.5) -- (3.3,1.5) -- (3.3,0.5) -- (3.8,0.5) -- (3.8,0) -- cycle;
    \node[blue!60, font=\small] at (1.2, 1.5) {Achievable Region};

    \begin{scope}
        \clip (0,0) rectangle (4,4);
        \fill[red!10, even odd rule] (0,0) rectangle (4,4)
            (0,0) -- (0,3) -- (2,3) -- (2,2.2) -- (2.2,2.2) -- (2.2, 2) -- (3,2) -- (3,1.5) -- (3.3,1.5) -- (3.3,0.5) -- (3.8,0.5) -- (3.8,0) -- cycle;
    \end{scope}
    \node[red!60, font=\small] at (2, 3.5) {Unachievable Region};

    \node[black!60, font=\small] at (1, 3.2) {\scriptsize Pareto Performance Frontier};

    \draw[very thin, gray!30] (0,0) grid (4,4);

    \draw[thick, ->] (-0.2,0) -- (4.2,0)
    node[midway, below, sloped, inner sep=0.5cm] {Objective 1};
    \draw[thick, ->] (0,-0.2) -- (0,4.2)
    node[midway, above, sloped, inner sep=0.5cm] {Objective 2};

    \foreach \x in {1,2,3,4} \draw (\x,1pt) -- (\x,-1pt) node[below, font=\footnotesize] {\x};
    \foreach \y in {1,2,3,4} \draw (1pt,\y) -- (-1pt,\y) node[left, font=\footnotesize] {\y};

    \draw[black, ultra thick, dashed] (0,3) -- (2,3) -- (2,2.2) -- (2.2,2.2) -- (2.2,2) -- (3,2) -- (3,1.5) -- (3.3,1.5) -- (3.3,0.5) -- (3.8,0.5) -- (3.8,0);

    \filldraw[black!60] (2,1) circle (2pt);
    \node[below, blue!80, font=\scriptsize] at (2,1) {$(2,1)$};

    \filldraw[black!60] (1.3,1.3) circle (2pt);
    \node[below, blue!80, font=\scriptsize] at (1.3,1.3) {$(1.3,1.3)$};



    \filldraw[black!60] (2.2,2) circle (2pt);
    \node[below, red!80, font=\scriptsize] at (2.2,2) {$(2.2,2)$};

    \filldraw[black!60] (2.5,1.5) circle (2pt);
    \node[below, blue!80, font=\scriptsize] at (2.5,1.5) {$(2.5,1.5)$};

    \filldraw[black!60] (1.5,2.5) circle (2pt);
    \node[below, blue!80, font=\scriptsize] at (1.5,2.5) {$(1.5,2.5)$};

    \filldraw[black!60] (0.5,2.7) circle (2pt);
    \node[below, blue!80, font=\scriptsize] at (0.5,2.7) {$(0.5,2.7)$};

    \filldraw[black!60] (0.9,2.1) circle (2pt);
    \node[below, blue!80, font=\scriptsize] at (0.9,2.1) {$(0.9,2.1)$};

    \filldraw[black!60] (0.5,0.8) circle (2pt);
    \node[below, blue!80, font=\scriptsize] at (0.5,0.8) {$(0.5,0.8)$};

    \filldraw[black!60] (1.5,0.5) circle (2pt);
    \node[below, blue!80, font=\scriptsize] at (1.5,0.5) {$(1.5,0.5)$};

    \filldraw[black!60] (2.5,0.2) circle (2pt);
    \node[above, blue!80, font=\scriptsize] at (2.5,0.2) {$(2.5,0.2)$};

    \filldraw[black!60] (2.6,0.8) circle (2pt);
    \node[below, blue!80, font=\scriptsize] at (2.6,0.8) {$(2.6,0.8)$};

    \filldraw[black!60] (2.2,2.2) circle (2pt);
    \node[above right, red!80, font=\scriptsize, inner sep=1pt] at (2.2,2.2) {$(2.2,2.2)$};

    \filldraw[black!60] (3.3,1) circle (2pt);

    \node[left, red!80, font=\scriptsize, xshift=2pt] at (3.2,1) {$(3.3,1)$};

    \filldraw[black!60] (3.3,1.5) circle (2pt);
    \node[above, red!80, font=\scriptsize, xshift=2pt] at (3.6,1.5)  {$(3.3,1.5) $};

    \filldraw[black!60] (3.8,0.5) circle (2pt);
    \node[above, red!80, font=\scriptsize, xshift=2pt] at (3.6,0.5)  {$(3.8,0.5) $};
    
    \filldraw[black!60] (2,3) circle (2pt);
    \node[right, red!80, font=\scriptsize] at (2,3) {$(2,3)$};

    \filldraw[black!60] (1.2,3) circle (2pt);
    \node[below, red!80, font=\scriptsize] at (1.2,3) {$(1.2,3)$};

    \filldraw[black!60] (3,2) circle (2pt);
    \node[right, red!80, font=\scriptsize, xshift=2pt] at (3,2) {$(3,2)$};

    \matrix [
        draw,
        fill=white,
        matrix of nodes,
        anchor=north east,
        xshift=0.9cm,
        yshift=-0.1cm,
        row sep=1pt,
        column sep=3pt,
        inner sep=3pt,
        nodes={anchor=center, inner sep=2pt, font=\scriptsize, align=left},
        scale=0.9,
        transform shape
    ] at (current bounding box.north east) {
        \node[circle, fill=black!60, minimum size=5pt] {}; & \node[anchor=west] {Achievable value}; \\
        \draw[black, dashed, ultra thick] (-0.3,0) -- (0.3,0); & \node[anchor=west] {Pareto-Performance\\Frontier (PPF)}; \\
        \node[fill=blue!10, draw=none, minimum width=0.6cm, minimum height=0.3cm] {}; & \node[anchor=west] {Achievable Region}; \\
        \node[fill=red!10, draw=none, minimum width=0.6cm, minimum height=0.3cm] {}; & \node[anchor=west] {Unachievable}; \\
         & \node[anchor=west] {Region}; \\
    };

\end{tikzpicture}
         }
         \caption{Visualization of the Pareto performance Frontier (PPF). The distinct points represent achievable values. The dashed boundary line demarcates the PPF, separating the achievable region (blue) from the unachievable region (red).}
         \label{fig:ppf}
     \end{figure}
\end{example}

Next, we provide an example illustrating a fundamental limitation of linear scalarization: its inability to capture Pareto-optimal points located in non-convex regions of the frontier.

\begin{example}
    Consider the Pareto-optimal points $A = [2, 3]^\top$, $B = [2.2, 2.2]^\top$, and $C = [3, 2]^\top$, as shown in Figure~\ref{fig:ppf}. We assert that for any preference vector $w \in \mathcal{W}$, point $B$ can never be the maximizer of linear scalarization. Specifically, either $w^\top B \leq w^\top A$ or $w^\top B \leq w^\top C$. 
    
    We prove this by contradiction. Suppose there exists a preference $w = [w_1, w_2]^\top \in \mathcal{W}$ (where $w_1, w_2 \geq 0$) such that $w^\top B > w^\top A$ and $w^\top B > w^\top C$. Expanding and simplifying these inequalities yields:
    \begin{align*}
        2.2(w_1 + w_2) > 2w_1 + 3w_2 \quad &\implies \quad w_1 > 4w_2 \\
        2.2(w_1 + w_2) > 3w_1 + 2w_2 \quad &\implies \quad w_2 > 4w_1
    \end{align*}
    Combining these constraints results in $w_1 > 4w_2 > 16w_1$. Since $w_1$ and $w_2$ are non-negative, $w_1 > 16w_1$ is a clear contradiction. Consequently, varying the preferences and maximizing a linear scalarization will fail to capture the entire Pareto frontier, particularly missing point $B$.
\end{example}





While the $\text{PPF}$ defines performance limits, it lacks a mechanism for selecting specific optimal trade-offs. To address this, we use \textbf{Chebyshev scalarization}, which maps preferences to frontier points. Unlike linear weighted sums, which fail on non-convex boundaries, Chebyshev scalarization ensures that every Pareto-optimal point can be recovered. We formalize this in Theorem~\ref{thm:cheb_suff_nece} next:

\begin{theorem}[Sufficiency and Necessity of the weighted Chebyshev Scalarization]\label{thm:cheb_suff_nece}
Let $\V \subset \mathbb{R}_{\geq 0}^d$ 
be a set of achievable value vectors such that $\V\subseteq [0,{R_{\max}}]^d$, and let the $\ppf$ denote the set of Pareto performance frontier for vectors in $\V$. For a given preference $w\in \W$, let $O_w(v)$ denote the weighted Chebyshev scalarization of $v$. The following properties hold:
\begin{enumerate}[label=\thetheorem.\arabic*]
    \item \label{thm:cheb_nece} \textbf{Sufficiency:} For any preference $w$, the maximizer of the Chebyshev scalarization is on the Pareto performance frontier. Formally, if $V^* \in \arg\max\limits_{v \in \mathcal{V}} O_w(v)$, then $V^* \in \ppf$.
    \item \label{thm:cheb_suff}  \textbf{Necessity:}  For any Pareto-optimal vector $V^* \in \mathcal{V}$, there exists a preference $w \in \mathcal{W}$ such that $V^*$ is the unique maximizer of the scalarization. Formally if $V' \in \arg\max\limits_{V \in \mathcal{V}} O_w( V)$ then $V'=V^*$.
\end{enumerate}
\end{theorem}
\begin{proof}[Proof Sketch]
The intuition for both properties relies on the geometric constraints of the Chebyshev scalarization:

\begin{itemize}
    \item \textbf{Sufficiency (Proof by Contradiction):} If a maximizer of the Chebyshev objective were not on the Pareto performance frontier ($\ppf$), there would exist another achievable vector in $\mathcal{V}$ that is strictly greater along all dimensions. Because the Chebyshev objective scales strictly with the minimum weighted component, this dominating vector would yield a strictly larger objective value, contradicting the premise that the original vector was a maximizer.
    \item \textbf{Necessity:} For any given Pareto-optimal vector $V^*$, we can perfectly align the preference vector $w \in \mathcal{W}$ by setting it proportional to $V^*$, i.e., $w:=\frac{V^*}{\|V^*\|_2}$. Therefore, any vector $V'\in \arg\max\limits_{V\in \mathcal{V}}O_w(V)$ maximizing the objective along this preference must weakly dominate $V^*$, since 
    $\min\limits_{w_i>0}\left\{\frac{V'_i}{w_i}\right\} = \min\limits_{V^*_i>0}\left\{\frac{V'_i \|V^*\|_2}{V^*_i}\right\} \geq O_w(V^*)=\|V^*\|_2$, yielding $V'_i\geq V^*_i$ for all $i$. Since $V^*$ is Pareto-optimal by definition, no achievable vector can strictly dominate it, forcing the unique maximizer to be exactly $V^*$.
\end{itemize}
A detailed proof is provided in Appendix~\ref{app:thm:cheb_suff_nece}.
\end{proof}



\noindent Theorem~\ref{thm:cheb_suff_nece} establishes the connection between Chebyshev scalarization and Pareto-optimality. As a consequence, the following Theorem
\ref{item:coverage} demonstrates that the set of scalarization maximizers covers all Pareto-optimal points while remaining entirely confined to the Pareto performance frontier.

However, to isolate the Pareto-optimal subset, we must filter out points that lie on the $\ppf$ but are not Pareto-optimal.
To achieve this, we introduce a two-stage optimization procedure in Theorem~\ref{item:exactness}. 
This procedure refines the selection logic: while the first stage uses Chebyshev scalarization to identify candidate solutions, the second stage imposes a maximum $\ell_2$-norm tie-breaking rule.
By selecting the vector with the largest $\ell_2$-norm, we effectively push the solution toward the most distal point on the $\ppf$, filtering out suboptimal frontier points.
This ensures the resulting points are strictly Pareto-optimal, thus exactly capturing the most efficient trade-offs.



\begin{theorem}[Coverage and Exact Characterization of Pareto Optimality]
\label{thm:pareto_coverage_characterization}
Let $\mathcal{V} \subset \mathbb{R}_{\geq 0}^d$ be a set of achievable value vectors, with $\mathcal{V}^{PO}$ denoting the Pareto-optimal set and $\ppf$ the Pareto performance frontier.

The relationship between Chebyshev scalarization and the Pareto set is characterized by the following two properties:

\begin{enumerate}[label=\thetheorem.\arabic*] 
    \item \textbf{Scalarization-Induced Coverage:} \label{item:coverage}
    Let $\mathcal{S}$ be the set constructed by selecting, for every preference $w \in \mathcal{W}$, an arbitrary vector that maximizes the scalarization $O_w$. Formally:
    \[
    \mathcal{S} \;:=\; \bigcup_{w \in \mathcal{W}} \big\{ v_w \big\}, \quad \text{where } v_w \in \arg\max_{v \in \mathcal{V}} O_w(v).
    \]
    Then $\mathcal{S}$ contains all Pareto-optimal points and consists solely of points on the Pareto performance frontier:
    \[ \mathcal{V}^{PO} \subseteq \mathcal{S} \subseteq \ppf. \]

    \item \textbf{Exact Characterization via Norm-Regularization:} \label{item:exactness}
    Let $\mathcal{M}$ be the set constructed by selecting, for every preference $w \in \mathcal{W}$, a vector $v^*_w$ obtained via a two-stage optimization process: first maximizing the Chebyshev scalarization $O_w$, and then maximizing the $\ell_2$-norm among the resulting candidates for tie-breaking. Formally:
    \[
    \mathcal{M} \;:=\; \bigcup_{w \in \mathcal{W}} \big\{ v^*_w \big\},
    \]
    where $v^*_w$ is an arbitrary solution to the nested optimization:
    \[
    v^*_w \in \arg\max_{v} \left\{ \|v\|_2 \;\Big|\; v \in \arg\max_{u \in \mathcal{V}} O_w(u) \right\}.
    \]
    Then $\mathcal{M}$ coincides with the Pareto-optimal set:
    \[ \mathcal{M} = \mathcal{V}^{PO}. \]
\end{enumerate}
\end{theorem}
\begin{proof}[Proof Sketch]

\begin{enumerate}[label=\arabic*., leftmargin=*]
    \item \textbf{Scalarization-Induced Coverage:} This property is a direct consequence of the Sufficiency and Necessity of the Chebyshev scalarization established in Theorem~\ref{thm:cheb_suff_nece}. Sufficiency guarantees that any maximizer chosen for $\mathcal{S}$ naturally falls on the Pareto performance frontier ($\mathcal{S}\subseteq \ppf)$. Necessity ensures that every true Pareto-optimal point is the \textit{unique} maximizer for its perfectly aligned preference vector, guaranteeing its inclusion in the coverage set ($\mathcal{V}^{PO}\subseteq \mathcal{S}$).
    
    \item \textbf{Exact Characterization:} To prove $\mathcal{M} = \mathcal{V}^{PO}$, we must show that the $\ell_2$-norm tie-breaker correctly filters out dominated points on the performance frontier. If the two-stage process selected a point that is not Pareto-optimal, a dominating Pareto-optimal point would exist. This dominating point would achieve at least the same scalarization score (passing the first stage) but possess a strictly larger $\ell_2$-norm, thereby winning the tie-breaker and contradicting the selection of the inferior point. Conversely, Pareto-optimal points are unique maximizers under their aligned preferences, meaning they trivially pass the tie-breaking phase.
\end{enumerate}
See Appendix~\ref{app:thm:pareto_coverage_characterization} for a detailed proof.
\end{proof}



\section{Bellman Update and Convergence}\label{sec:bellman_update}

\noindent \textbf{Main Idea:} Theorem~\ref{item:exactness} allows us to characterize the set of Pareto-optimal values via preferences and the two-stage optimization problem involving Chebyshev scalarization in the first stage and norm-maximization in the second stage.
By restricting the domain to value vectors \textit{achievable} by policies, we  characterize the set of Pareto-optimal $Q$-values obtained by starting from a state $s$, taking an action $a$, and subsequently behaving Pareto-optimally with respect to a preference $w$.
We formalize this concept in Definition~\ref{def:optimal_pref_act_val_fun}.

To formulate the preference-conditioned Bellman operator, we leverage the coverage guarantee from Theorem~
\ref{item:coverage}.
Crucially, the Bellman update utilizes only the maximization procedure involving the Chebyshev scalarization, since that is sufficient to ensure coverage.
Building on this formulation, we prove that the operator satisfies the \textbf{enveloping property}: the estimated value vectors upper-bound the Pareto-optimal values across the preference space.
We demonstrate that this property is \textbf{invariant} under the Bellman update: if the initialization envelopes the frontier, subsequent updates maintain this upper bound while monotonically approaching the Pareto performance frontier from above.
This monotonic convergence guarantees that the final value estimates cover all Pareto-optimal values.

Finally, to isolate the Pareto-optimal subset, we apply the two-stage optimization procedure to the converged $Q$-value vectors.
To recover the corresponding Pareto-optimal policies, we rely on Theorem~\ref{thm:rec_decomp}, which establishes that Pareto-optimal values are recursive.
This recursivity allows us to recover the optimal policies by acting greedily with respect to the converged $Q$-values.

Now we define the optimal preference-action-value function as follows.

\begin{definition}[Optimal Preference Action-Value Function]
\label{def:optimal_pref_act_val_fun}
For a preference vector $w \in \mathcal{W}$, the \emph{Optimal Preference Action-Value Function} $\mathbf{Q}^* : \mathcal{S} \times \mathcal{A} \times \mathcal{W} \rightarrow \mathbb{R}^d$
maps a state $s \in \mathcal{S}$, an action $a \in \mathcal{A}$, and a preference vector $w$ to the expected vector-valued return obtained by taking action $a$ in state $s$ and subsequently following an optimal policy. It is defined as
\begin{equation*}
\mathbf{Q}^*(s,a,w)
:= \mathbf{R}(s,a) + \gamma \, \mathbf{V}^*_w\big(p(\cdot \mid s,a)\big),
\end{equation*}
where the value function $\mathbf{V}^*_w$ is determined via a \textbf{two-stage optimization procedure}:
\begin{enumerate}[leftmargin=*, nosep]
    \item \textbf{Chebyshev optimization:} Identify the  policies that maximize the weighted Chebyshev scalarization $O_w(\cdot)$.
    \item \textbf{Tie-breaking:} Among the resulting optimal value vectors, select one with maximum $\ell_2$-norm.
\end{enumerate}
Formally, this is expressed as
\begin{equation*}
\begin{split}
    \mathbf{V}^*_w(p(\cdot \mid s,a)) \in \arg\max_{V} \Big\{ \|\mathbf{V}\|_2 \;\big|\; \mathbf{V} = \mathbf{V}^\pi(p(\cdot \mid s,a)), \\
    \pi \in \arg\max_{\pi' \in \Pi} O_w(\mathbf{V}^{\pi'}(p(\cdot \mid s,a))) \Big\}.
\end{split}
\end{equation*}
\end{definition}

As proved in Theorem
~\ref{item:exactness}, this two-stage optimization allows us to characterize the set of Pareto-optimal points via preferences.
The following theorem establishes that the preference-conditioned optimal action-value function $\mathbf{Q}^*$ admits a recursive decomposition.


\begin{theorem}[Recursive Decomposition of $Q^*$]\label{thm:rec_decomp}
Let $\mathbf{Q}^*: S \times A \times \mathcal{W} \to \mathbb{R}^d_{\geq 0}$ be the optimal preference action-value function. For any state $s \in S$, action $a \in A$, and preference vector $w \in \mathcal{W}$, $\mathbf{Q}^*$ satisfies the recursive relationship:

\scalebox{0.95}{$\mathbf{Q}^*(s, a, w) = R(s, a) + \gamma \E\limits_{s' \sim p(\cdot \mid s, a)} \left[ \mathbf{Q}^*(s', \hat{A}(s'), \hat{\mathcal{W}}(s')) \right]$}

where the functions $\hat{A}: S \to A$ and $\hat{\mathcal{W}}: S \to \mathcal{W}$ map the next-states to actions and preferences, respectively.
\end{theorem}
\begin{proof}[Proof Sketch] Let $\pi$ be a policy achieving $\mathbf{Q}^*(s,a,w)$. The tail policy induced by $\pi$ at every next state $s'$ must be Pareto-optimal. If it were not, there would exist a strictly dominating policy $\pi'$ at $s'$. We could then construct a stitched policy $\pi''$ that follows $\pi'$ at $s'$ and $\pi$ otherwise. This policy $\pi''$ would strictly dominate $\pi$ at $(s,a)$ (due to additive nature of expectation), which contradicts the Pareto-optimality of $\pi$ achieving $Q^*(s,a,w)$. Thus, the tail value of $\pi$ at $s'$ must itself be the Pareto-optimal.
Theorem~\ref{thm:pareto_coverage_characterization} guarantees the existence of a preference $w'$ recovering this value. 
By setting $\hat{A}(s') = \pi(s')$ and $\hat{\mathcal{W}}(s') = w'$, we recover the recursive form in terms of $Q^*(s',\hat{A}(s'),\hat{\mathcal{W}}(s'))$. (See Appendix~\ref{app:thm:rec_decomp} for details.)
\end{proof}

This recursive decomposition demonstrates that the Pareto-optimal action-value function $Q^*$ is greedy with respect to the current preference $w$. 
Specifically, after taking an action $a$ from the current state $s$, the agent can optimally navigate the environment by greedily selecting the subsequent action $\hA(s')$ and preference $\W(s')$ at each successive state $s'$. This ensures that the expected vector-valued rewards are maximized along the original preference $w$ thus capturing the Pareto-optimal trade-offs.

Leveraging the recursive structure established in Theorem \ref{thm:rec_decomp}, we introduce the Preference Bellman Operator $\T$ (Definition \ref{def:pref_bell_operator}) to formalize the iterative update process required to recover the Pareto-optimal values. The relationship between the iterative application of this operator and the Pareto performance frontier is established in Theorem \ref{thm:bellman_env_conv}. Specifically, we prove that, when initialized with an upper-bounding $Q$-function, the sequence of functions generated by the repeated application of $\T$ monotonically converges to a coverage set of Pareto-optimal values on the $\ppf$.



 

\begin{definition}[Preference Bellman Optimality Operator]\label{def:pref_bell_operator}
Let $\mathcal{Q}$ be the space of bounded vector-valued functions mapping $\mathcal{S} \times \mathcal{A} \times \mathcal{W} \to \mathbb{R}^d_{\geq 0}$. We define the \textit{Preference Bellman Optimality Operator} $\mathcal{T}: \mathcal{Q} \to \mathcal{Q}$ as follows. 

For any $Q \in \mathcal{Q}$, state $s \in \mathcal{S}$, action $a \in \mathcal{A}$, and preference vector $w \in \mathcal{W}$:

\scalebox{0.9}{
$(\mathcal{T} Q)(s, a, w) := R(s, a) + \gamma \E_{s'\sim p(\cdot\mid s,a)}\Big[ Q\big(s', \hat{a}(s'), \hat{w}(s')\big) \Big],$
}
where the next-state action $\hat{a}(s')$ and preference $\hat{w}(s')$ are chosen greedily to maximize the scalarized value at the next state:
\[
\scalebox{0.95}{
$\displaystyle
(\hat{a}, \hat{w}) \in \operatorname*{arg\,max}_{\substack{\phi_A:S\rightarrow A \\ \phi_\W : S\rightarrow \W }} O_w\Big( \E_{s'\sim p(\cdot\mid s,a)} \big[Q(s', \phi_A(s'), \phi_\W(s'))\big]\Big)
$}
\]
\end{definition}
\noindent The Preference Bellman Operator $\mathcal{T}$ extends the standard notion of optimality to the multi-objective setting by treating the preference vector as a dynamic component of the decision process. Rather than simply evaluating an action's immediate reward, the operator performs a \textit{preference-conditioned backup}.

When evaluating the state-action pair $(s,a)$ under a current preference $w$, the operator looks ahead to the next state $s'$ and identifies not only the optimal next action $\hat{a}$ but also the optimal future preference $\hat{w}$. This $\hat{w}$ acts as a steering mechanism: it selects the specific point on the next state's frontier that contributes most effectively to maximizing the scalarized return for the \textit{current} preference $w$.

Next, we prove the convergence of the Bellman operator in Theorem~\ref{thm:bellman_env_conv}.

\begin{theorem}\label{thm:bellman_env_conv}
    Let us define $Q(s,a,w):= \R$, then for every $n\in \mathbb{N}$, and each state $s$, action $a$, the following three conditions hold:
    \begin{enumerate}[label=\thetheorem.\arabic*]
        \item \label{prop:opt_coverage} \textbf{Upper-bound}:  Every Pareto-optimal value is upper-bounded by an estimate. Specifically, for every preference $w$ there exists a preference $w'$ such that: $$Q^*(s,a,w)\preceq (\T^nQ)(s,a,w')$$
        \item \label{prop:pareto_envelope} \textbf{Pareto Envelope}: The estimates lie on or above the Pareto performance frontier. Formally, for all preferences $w$ the estimate is never dominated in all components by a Pareto-optimal value, i.e., \begin{gather*}
            \forall w, \nexists w' \text{ such that } \\
            (\boldsymbol{\mathcal{T}}^nQ)(s,a,w)_i < Q^*(s,a,w')_i, \quad \forall i=1,\dots,d
        \end{gather*}
        \item \label{prop:asymp_real} \textbf{Asymptotic Convergence}: Every estimate converges to the value of a realizable policy. Formally, for all preferences $w$ there is a policy $\pi \in \Pi$ such that the gap between the estimate and the policy's value vanishes: $$\bm{0}\preceq (\T^nQ)(s,a,w) - V^\pi(s) \preceq \gamma^n \R$$
    \end{enumerate}
\end{theorem}

\begin{proof}[Proof Sketch]
\textbf{Upper-bound:} We proceed by induction. The base case holds by the initialization $Q \succeq Q^*$. For the inductive step, we invoke the Recursive Decomposition theorem (Theorem~\ref{thm:rec_decomp}), which expresses $Q^*$ as an immediate reward plus a discounted next-state value. By the inductive hypothesis, the next-state $Q^*$ is upper-bounded by the previous iterate $\mathcal{T}^{n-1}Q$. Since the operator $\mathcal{T}$ performs a maximization over potential next-state updates, the update using the upper-bounding estimate yields a value at least as high as $Q^*$, preserving the inequality $Q^*(s,a,w)\preceq (\mathcal{T}^nQ)(s,a,w')$.

\textbf{Pareto Envelope:} We argue by contradiction. Suppose there exist preferences $w$ and $w'$ such that the estimate is strictly less than an optimal value in all components: $(\mathcal{T}^nQ)(s,a,w)_i < Q^*(s,a,w')_i$ for all $i=1,\dots,d$. 
From the Upper-bound property (Theorem~\ref{prop:opt_coverage}), we know $Q^*(s,a,w')$ is itself upper-bounded by some estimate $\mathcal{T}^nQ(s,a,w'')$. By transitivity, this implies $(\mathcal{T}^nQ)(s,a,w)_i < (\mathcal{T}^nQ)(s,a,w'')_i$ for all $i$. In other words, the estimate for $w''$ is strictly greater than the estimate for $w$ in every component. 
Subtracting $R(s,a)$ and dividing by $\gamma$ implies that the expected continuation value used in the update for $w''$ is strictly greater component-wise than that for $w$. Since every component is strictly larger, the resulting Chebyshev scalarization is strictly higher. This contradicts the operator's maximization step, which would have selected the higher update.

\textbf{Asymptotic Convergence:} The value $\mathcal{T}^nQ(s,a,w)$ represents the accumulation of rewards from $n$ steps of optimal updates plus a residual initialization term bounded by $\gamma^n \R$.
We construct a deterministic, non-stationary, history-dependent policy $\pi$ that exactly mimics the sequence of action updates chosen by the operator $\mathcal{T}$ for the first $n$ steps, and acts arbitrarily thereafter. 
The value of this policy $V^\pi(s)$ matches the estimate $\mathcal{T}^nQ(s,a,w)$ exactly in the first $n$ terms, with the difference arising solely from the tail rewards after step $n$. 
Since rewards are bounded, this difference is bounded by $\gamma^n \mathbb{R}$, which vanishes as $n \to \infty$.

(See Appendix~\ref{app:sec:conv_proof} for the detailed proof.)
\end{proof}


\textbf{Discussion of Theorem~\ref{thm:bellman_env_conv}}: The significance of this theorem lies in its guarantee that the preference-conditioned Bellman operator monotonically converges toward the Pareto performance frontier. Unlike standard value iteration, which converges to a single fixed point, our operator converges to a set of values that fully characterizes the frontier.

Condition~\ref{prop:opt_coverage} (Upper-bound) acts as a safety guarantee, ensuring that the iterative updates never discard the Pareto-optimal values; the estimates remain sufficiently expressive to cover all optimal trade-offs.

Condition~\ref{prop:pareto_envelope} (Pareto Envelope) establishes the geometric relationship of the updates: the estimates approach the frontier strictly from the non-dominated side, creating an ``envelope" that tightens with every iteration. 

Condition~\ref{prop:asymp_real} (Asymptotic Convergence) bridges the gap between estimation and execution. It ensures that as $n\rightarrow \infty$, the gap between the estimates and achievable values vanishes, guaranteeing that the limiting values of the sequence represent executable policies.

In the next section, we discuss how to obtain Pareto-optimal policies along a preference $w$.

\section{Obtaining Pareto-optimal Policies}
\label{sec:policy}


The convergence and recursive properties of the Preference Bellman Operator $\mathcal{T}$ enable us to recover policies that approximate all Pareto-optimal values. Crucially, as established by the recursive decomposition in Theorem~\ref{thm:rec_decomp}, these policies can be executed greedily. 
Rather than tracking the entire execution history, the agent requires only a constant-size memory consisting of the most recent transition data, namely the previous state $s_{i-1}$, action $a_{i-1}$, and reward $r_{i-1}$, alongside the current guiding preference $w_i$.
At step $i$, upon observing the current state $s_i$, the agent greedily selects an action-preference pair $(a_i, w_{i+1})$ based on the Q-estimates.
This pair dictates the immediate action to execute and the updated preference parameter required to guide the policy from the subsequent state. The memory is then shifted forward accordingly. Algorithm~\ref{alg:high_level_pareto_policy} provides the high-level pseudocode for this execution, starting from an initial state $s_0$ and initial preference $w_{init}$ (with the detailed procedure, Algorithm~\ref{alg:pareto_execution}, deferred to Appendix~\ref{app:sec:approx_POC}). Next, we establish the theoretical guarantees for the policies generated by this algorithm.

\begin{algorithm}[tb]
  \caption{High-Level Pareto-Optimal Policy Execution}
  \label{alg:high_level_pareto_policy}
  \begin{algorithmic}[1]
    \REQUIRE Value estimate $\mathcal{T}^n Q$, Next state distribution $p(\cdot \mid s,a)$, Initial state $s_0$, Initial Preference $w_{\text{init}}$
    
    \STATE \textbf{Initialization:} Compute an initial action $a_0$ and next preference $w_1$ that maximize the Chebyshev scalarization along preference $w_{init}$ of  Pareto-optimal estimates $\mathcal{T}^n Q(s_0, a_0, w_1)$.
    \STATE Take action $a_0$ in the environment and observe the reward $r_0$.
    \STATE Set the history: $s_{prev} \leftarrow s_0$, $a_{prev} \leftarrow a_0$, reward $r_{prev} \leftarrow r_0$, and next preference target $w_{target} \leftarrow w_1$.
    
    \WHILE{not termination}
        \STATE \textbf{Alignment:} Align the target preference $w_{target}$ with the direction of the expected tail reward estimate.
        \STATE \textbf{Optimization:} For the current state $s_i$, obtain the optimal action $a_i$ and next preference $w_{i+1}$ that maximize the Chebyshev scalarization of the expected reward estimates along the aligned preference $w_{target}$.
        \STATE \textbf{Execution:} Take the optimal action $a_i$ in the environment and observe the reward $r_i$.
        \STATE \textbf{Update:} Update the history with the new state $s_{prev} \leftarrow s_i$, action $a_{prev} \leftarrow a_i$, reward $r_{prev} \leftarrow r_i$, and next preference target $w_{target} \leftarrow w_{i+1}$.
    \ENDWHILE
  \end{algorithmic}
\end{algorithm}

\begin{theorem}[Approximate Pareto Coverage]\label{thm:app_pareto_coverage}
    Let $\mathcal{C}^n_{PO} = \{ \pi_w \mid w \in \mathcal{W} \}$ denote the set of policies generated by all preferences using the $n$-th Bellman update by Algorithm~\ref{alg:high_level_pareto_policy}. This set satisfies the properties of approximate Pareto coverage. 
    

    \begin{enumerate}[label=\thetheorem.\arabic*]
        \item \textbf{Approximate Coverage:} 
        Every Pareto-optimal policy is approximated by some policy in $\mathcal{C}^n_{PO}$ such that the \textbf{suboptimality gap} is bounded. Formally:
        \[
        \begin{aligned}
            &\forall \pi \in \Pi, \exists w \in \mathcal{W} \text{ such that } \\
            & V^{\pi}(s_0) \preceq V^{\pi_w}(s_0) + \gamma^n \mathcal{R}
        \end{aligned}
        \]

        \item \textbf{Approximate Parsimony:} 
        Every policy in $\mathcal{C}^n_{PO}$ satisfies \textbf{$\bm\gamma^{\bm n}\bm\R$-Pareto optimality}; that is, no other policy can strictly dominate it by a margin larger than the approximation bound. Formally:
        \[
        \begin{aligned}
            &\forall w \in \mathcal{W}, \forall \pi \in \Pi \quad \text{ we get} \\
            &V^{\pi}(s_0) \nsucc V^{\pi_w}(s_0) + \gamma^n \mathcal{R}
        \end{aligned}
        \]
    
    \end{enumerate}

\end{theorem}

\begin{proof}[Proof Sketch]
The argument depends on bounding the performance gap between the theoretical $n$-step Bellman estimate $\mathcal{T}^n Q(s,a,w)$ and the true expected value $V^{\pi_w}(s)$ of the policy executing the greedy unrolling procedure starting from preference $w$.

\begin{enumerate}[label=\arabic*., leftmargin=*]
    \item \textbf{Recursive Trajectory Bound:} At each step of execution, the agent greedily re-plans using the full $n$-step horizon aligned with its future expected return. By unrolling the recursive error over the infinite horizon, the maximum performance drop is bounded by the residual term:
    $$\mathcal{T}^n Q(s,a,w) \preceq V^{\pi_w}(s) + \gamma^n \mathcal{R}$$
    
    \item \textbf{Approximate Coverage:} We know from the fundamental properties of the scalarized Bellman operator that the value of any true Pareto-optimal policy $\pi^*$ is weakly dominated by the theoretical $n$-th Bellman estimate for \textit{some} specifically aligned preference $w'$ added to a small residual term $\gamma^n\R$. Chaining this bound with our execution bound guarantees that the induced policy $\pi_{w'}$ approximates the true optimal policy $\pi^*$ within the residual gap:
    $$V^{\pi^*}(s) \preceq \mathcal{T}^n Q(s,a,w') \preceq V^{\pi_{w'}}(s) + \gamma^n \mathcal{R}$$
    
    \item \textbf{Approximate Parsimony:} Since the estimates generated by the Bellman operator cannot be strictly dominated by any policy, and the executed policy $\pi_w$ successfully recovers this estimate up to the $\gamma^n \mathcal{R}$ residual margin, no alternative policy exists that can strictly dominate $V^{\pi_w}(s)$ by a margin larger than $\gamma^n \mathcal{R}$.
\end{enumerate}
A detailed proof is provided in Appendix~\ref{app:thm:app_pareto_coverage}.
\end{proof}


In summary, our framework guarantees that by utilizing the Chebyshev-based Bellman operator to converge to accurate Q-value estimates, and subsequently filtering these estimates to break ties, we can extract an approximate set of Pareto-optimal policies.


\section{Experiments}\label{sec:experiments}
In this section, we empirically validate the theoretical claims presented in Section~\ref{sec:bellman_update}. Our evaluation addresses three primary questions:
\begin{enumerate}[label=\textbf{Q\arabic*.}]
    \item \textbf{Convergence:} Does the recursive application of the preference-aware Bellman operator result in the convergence of value estimates to the Pareto frontier?
    \item \textbf{Coverage:} Do the converged estimates successfully cover the entire set of Pareto-optimal trade-offs?
    \item \textbf{Policy Extraction:} Do the synthesized deterministic policies achieve the estimated Pareto-optimal values (solving the \textit{alignment consistency} problem)?
\end{enumerate}
To answer these, we evaluate our algorithm on the standard \textit{Deep Sea Treasure} (convex frontier) and a modified \textit{Deep Sea Treasure Concave} variant~\cite{deep-sea-benchmark}. We note that Pareto-optimal values and policies are known for both the benchmarks \cite{deep-sea-benchmark}. 
While the Deep-Sea-Treasure has a convex Pareto front, the Deep-Sea-Treasure-Concave has a concave Pareto front.

\begin{figure}[t]
    \centering
    \resizebox{0.6\linewidth}{!}{
\begin{tikzpicture}

\definecolor{darkgray176}{RGB}{176,176,176}
\definecolor{green}{RGB}{0,128,0}
\definecolor{lightgray204}{RGB}{204,204,204}
\definecolor{orange}{RGB}{255,165,0}

\begin{axis}[
axis lines=box,
tick align=outside,
tick pos=left,
title={Deep-Sea-Treasure Concave},
xlabel={Objective 1},
ylabel={Objective 2},
xmin=-28.10229595, xmax=207.71915695,
ymin=-8.9105, ymax=208.9005,
xtick style={color=black},
ytick style={color=black},
legend cell align={left},
legend style={
  at={(0.5,-0.18)},
  anchor=north,
  legend columns=4,
  row sep=2pt,
  column sep=6pt,
  font=\small,
  fill opacity=0.8,
  draw=lightgray204
}
]
\addplot [draw=blue, fill=blue, mark size=3pt, mark=*, only marks]
table{%
x  y
197 198
197 199
};
\addlegendentry{$\hat{Q}_{1}$}
\addplot [draw=red, fill=red, mark size=3pt, mark=*, only marks]
table{%
x  y
63.826202 87.238396
55.09134 175.95421
65.68101 77.18188
64.74893 79.931
69.50438 74.434685
71.47449 73.47449
59.34906 116.350525
60.226574 93.74771
56.768784 135.48236
67.57349 75.35627
};
\addlegendentry{$\hat{Q}_{100}$}
\addplot [draw=green, fill=green, mark size=3pt, mark=triangle, only marks]
table{%
x  y
-0.991452 1.008548
-4.8924475 2.8903363
-12.239351 21.281786
-17.37459 103.48827
-13.116871 43.884605
-6.784917 4.715949
-8.639728 14.772466
-7.716983 7.4650717
-2.9615521 1.968748
-15.697134 63.016434
};
\addlegendentry{$\hat{Q}_{1000}$}
\addplot [draw=orange, fill=orange, mark size=3pt, mark=square, only marks]
table{%
x  y
-7.72553 7.4565234
-8.648275 14.763917
-2.9700997 1.9602005
-0.99999964 1.0000004
-15.705681 63.007885
-12.247897 21.273237
-13.125419 43.876053
-17.383139 103.479706
-4.900995 2.8817885
-6.7934647 4.7074013
};
\addlegendentry{$\hat{Q}_{2000}$}
\addplot [draw=blue, fill=blue, mark size=3pt, mark=x, only marks]
table{%
x  y
-1 1
-1.99 0.99
};
\addlegendentry{$\pi_{1}$}
\addplot [draw=red, fill=red, mark size=3pt, mark=x, only marks]
table{%
x  y
-2.9701 1.9602
-8.64827525 14.76391511
-1 1
-6.79346521 4.70740075
-12.2478977 21.27323692
-15.70568066 63.00787506
-7.72553056 7.45652278
-17.38313762 103.47970642
-13.12541872 43.87605115
};
\addlegendentry{$\pi_{100}$}
\addplot [draw=green, fill=green, mark size=3pt, mark=x, only marks]
table{%
x  y
-2.9701 1.9602
-8.64827525 14.76391511
-1 1
-6.79346521 4.70740075
-12.2478977 21.27323692
-15.70568066 63.00787506
-4.90099501 2.88178803
-7.72553056 7.45652278
-17.38313762 103.47970642
-13.12541872 43.87605115
};
\addlegendentry{$\pi_{1000}$}
\addplot [draw=orange, fill=orange, mark size=3pt, mark=x, only marks]
table{%
x  y
-2.9701 1.9602
-8.64827525 14.76391511
-1 1
-6.79346521 4.70740075
-12.2478977 21.27323692
-15.70568066 63.00787506
-4.90099501 2.88178803
-7.72553056 7.45652278
-17.38313762 103.47970642
-13.12541872 43.87605115
};
\addlegendentry{$\pi_{2000}$}
\end{axis}

\end{tikzpicture}
    }

    \caption{\textbf{Coverage of Non-Convex Regions.} In the concave variant, linear methods would fail to find the points in the locally concave ``indented" region. Our method recovers the complete frontier, including non-convex trade-offs.}
    \label{fig:dst_concave}
\end{figure}
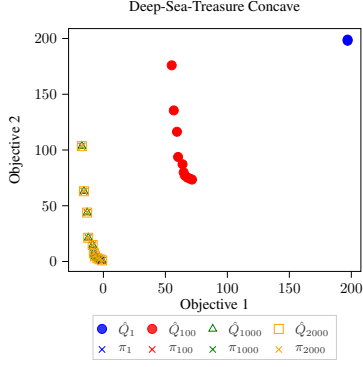

\begin{figure}[t]
    \centering
    \resizebox{0.6\linewidth}{!}{
\begin{tikzpicture}

\definecolor{darkgray176}{RGB}{176,176,176}
\definecolor{green}{RGB}{0,128,0}
\definecolor{lightgray204}{RGB}{204,204,204}
\definecolor{orange}{RGB}{255,165,0}

\begin{axis}[
tick align=outside,
tick pos=left,
title={Deep-Sea-Treasure},
xlabel={},
xmin=-28.10229595, xmax=207.71915695,
xtick style={color=black},
y grid style={darkgray176},
ylabel={},
ymin=-9.20735, ymax=208.60035,
ytick style={color=black}
]
\addplot [draw=blue, fill=blue, mark size=3pt, mark=*, only marks]
table{%
x  y
197 198
197 198.7
};

\addplot [draw=red, fill=red, mark size=3pt, mark=*, only marks]
table{%
x  y
64.74893 86.54867
56.768784 91.54713
71.47449 73.174484
55.09134 92.25245
63.826202 87.330666
59.34906 90.28816
69.50438 80.5113
67.57349 83.521324
60.226574 89.84762
65.68101 85.655205
};

\addplot [draw=green, fill=green, mark size=3pt, mark=triangle, only marks]
table{%
x  y
-0.991452 0.708548
-13.116871 17.822226
-17.37459 19.786526
-2.9615521 8.045367
-8.639728 14.864738
-15.697134 19.081205
-7.716983 14.082736
-6.784917 13.189271
-4.8924475 11.055403
-12.239351 17.381695
};

\addplot [draw=orange, fill=orange, mark size=3pt, mark=square, only marks]
table{%
x  y
-8.648275 14.856192
-6.7934647 13.180723
-13.125419 17.813683
-17.383139 19.777979
-12.247897 17.373148
-7.72553 14.074188
-15.705681 19.072657
-2.9700997 8.036821
-4.900995 11.046855
-0.99999964 0.70000035
};

\addplot [draw=blue, fill=blue, mark size=3pt, mark=x, only marks]
table{%
x  y
-1 0.7
-1.99 0.693
};

\addplot [draw=red, fill=red, mark size=3pt, mark=x, only marks]
table{%
x  y
-4.90099501 11.04685411
-12.2478977 17.37314349
-2.9701 8.03682
-13.12541872 17.81367677
-6.79346521 13.18072209
-8.64827525 14.85618958
-1 0.7
-17.38313762 19.77797615
-15.70568066 19.07265407
};

\addplot [draw=green, fill=green, mark size=3pt, mark=x, only marks]
table{%
x  y
-4.90099501 11.04685411
-7.72553056 14.07418675
-12.2478977 17.37314349
-2.9701 8.03682
-13.12541872 17.81367677
-6.79346521 13.18072209
-8.64827525 14.85618958
-1 0.7
-17.38313762 19.77797615
-15.70568066 19.07265407
};

\addplot [draw=orange, fill=orange, mark size=3pt, mark=x, only marks]
table{%
x  y
-4.90099501 11.04685411
-7.72553056 14.07418675
-12.2478977 17.37314349
-2.9701 8.03682
-13.12541872 17.81367677
-6.79346521 13.18072209
-8.64827525 14.85618958
-1 0.7
-17.38313762 19.77797615
-15.70568066 19.07265407
};

\end{axis}

\end{tikzpicture}
    }\caption{\textbf{Convergence on Convex Frontier.} Value estimates (shapes) converge to the true Pareto frontier. The policies (crosses) overlap perfectly with the estimates, demonstrating accurate policy extraction.}
    \label{fig:dst_convex}
\end{figure}
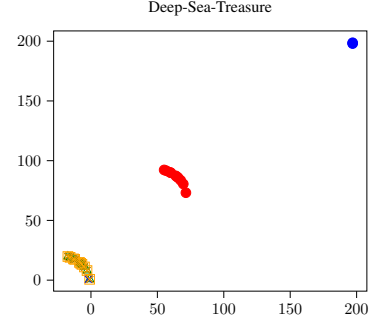



\textbf{Q1 Convergence:} Figures~\ref{fig:dst_concave} and \ref{fig:dst_convex} illustrate the evolution of the value estimates $\hat{Q}$ over training iterations $n=\{1, 100, 1000, 2000\}$. We observe that the estimates (represented by shapes) monotonically advance towards the frontier and stabilize after approximately at 1,000 steps. The negligible shift between the 1,000-step (Green) and 2,000-step (Orange) estimates affirms asymptotic convergence.

\textbf{Q2 Coverage:} The estimates $\hat{Q}$ for both environments align with the ground truth Pareto optimal values. 
Notably, in the Concave environment, our method successfully identifies all the trade-offs , which are typically skipped by linear scalarization methods that only recover the convex hull.

\textbf{Q3 Policy Extraction:} To verify consistency, we executed the extracted policies plotted the realized returns as crosses ($\times$). The perfect overlap between the predicted estimates ($\hat{Q}_{2000}$) and the actual policy returns ($\pi_{2000}$) confirms that our policies achieve the estimated Pareto-optimal trade-offs.

\textbf{Baselines:} We benchmark our results against the \textit{MORL-Baselines} project~\cite{felten_toolkit_2023}\footnote{\url{https://wandb.ai/openrlbenchmark/MORL-Baselines}}.
According to these benchmarks, \textbf{Pareto Q-Learning (PQL)}~\cite{van2014multi} successfully recovers the complete Pareto-optimal frontier in both the convex and concave Deep Sea Treasure environments.
In contrast, \textbf{Multi-Policy Multi-Objective Q-Learning (MP-MOQ)}~\cite{inproceedings_van}  recovers nearly the entire frontier in the convex setting, it fails significantly in the concave environment, identifying only the two points of the concave frontier.

\section{Conclusion}
\label{sec:conclusion}


We presented a framework for the \textit{Approximate Complete And Parsimonious Synthesis} ($\epsilon$-CAPS) of deterministic Pareto-optimal policies. By introducing a novel Bellman operator based on Chebyshev scalarization, we established theoretical guarantees for asymptotic convergence to the Pareto performance frontier. Crucially, we demonstrated how to obtain deterministic policies from the converged Q-estimates, ensuring that every recovered policy is approximately non-dominated, and every non-dominated value is approximately recovered. Empirically, our method successfully captures complex, non-convex trade-offs where standard linear scalarizations strictly fail. Future work will focus on scaling this formulation to high-dimensional state spaces via deep reinforcement learning architectures.





\bibliography{example_paper}

@article{white1982multi,
  title={Multi-objective infinite-horizon discounted Markov decision processes},
  author={White, DJ},
  journal={Journal of Mathematical Analysis and Applications},
  volume={89},
  number={2},
  pages={639--647},
  year={1982},
  publisher={Elsevier}
}

@article{DBLP:journals/anor/Mifrani25,
  author       = {Anas Mifrani},
  title        = {A counterexample and a corrective to the vector extension of the Bellman
                  equations of a Markov decision process},
  journal      = {Ann. Oper. Res.},
  volume       = {345},
  number       = {1},
  pages        = {351--369},
  year         = {2025}
}

@inproceedings{inproceedings_van,
author = {Van Moffaert, Kristof and Drugan, Madalina and Nowe, Ann},
year = {2013},
month = {04},
pages = {},
title = {Scalarized Multi-Objective Reinforcement Learning: Novel Design Techniques},
journal = {IEEE Symposium on Adaptive Dynamic Programming and Reinforcement Learning, ADPRL},
doi = {10.1109/ADPRL.2013.6615007}
}

@inproceedings{Yamamoto2019HypervolumeBasedMR,
	address = {Berlin, Heidelberg},
	author = {Van Moffaert, Kristof and Drugan, Madalina M. and Now{\'e}, Ann},
	booktitle = {Evolutionary Multi-Criterion Optimization},
	editor = {Purshouse, Robin C. and Fleming, Peter J. and Fonseca, Carlos M. and Greco, Salvatore and Shaw, Jane},
	isbn = {978-3-642-37140-0},
	pages = {352--366},
	publisher = {Springer Berlin Heidelberg},
	title = {Hypervolume-Based Multi-Objective Reinforcement Learning},
	year = {2013}}

@inproceedings{10.1007/978-3-030-22649-7_25,
author = {Yamaguchi, Tomohiro and Nagahama, Shota and Ichikawa, Yoshihiro and Takadama, Keiki},
title = {Model-Based Multi-objective Reinforcement Learning with Unknown Weights},
year = {2019},
isbn = {978-3-030-22648-0},
publisher = {Springer-Verlag},
address = {Berlin, Heidelberg},
url = {https://doi.org/10.1007/978-3-030-22649-7_25},
doi = {10.1007/978-3-030-22649-7_25},
booktitle = {Human Interface and the Management of Information. Information in Intelligent Systems: Thematic Area, HIMI 2019, Held as Part of the 21st HCI International Conference, HCII 2019, Orlando, FL, USA, July 26-31, 2019, Proceedings, Part II},
pages = {311–321},
numpages = {11},
location = {Orlando, FL, USA}
}

@article{article_Ruiz,
author = {Ruiz-Montiel, Manuela and Mandow, Lawrence and Pérez-de-la-Cruz, José-Luis},
year = {2017},
month = {06},
pages = {},
title = {A Temporal Difference Method for Multi-Objective Reinforcement Learning},
volume = {263},
journal = {Neurocomputing},
doi = {10.1016/j.neucom.2016.10.100}
}

@inproceedings{mandow2018pruning,
  title={Pruning dominated policies in multiobjective Pareto q-learning},
  author={Mandow, Lawrence and P{\'e}rez-de-la-Cruz, Jos{\'e}-Luis},
  booktitle={Conference of the Spanish Association for Artificial Intelligence},
  pages={240--250},
  year={2018},
  organization={Springer}
}

@inproceedings{Reymond2019ParetoDQNAT,
  title={Pareto-DQN: Approximating the Pareto front in complex multi-objective decision problems},
  author={Mathieu Reymond and Ann Now{\'e}},
  year={2019},
  url={https://api.semanticscholar.org/CorpusID:251245284}
}

@INPROCEEDINGS{4220828,
  author={Wiering, Marco A. and de Jong, Edwin D.},
  booktitle={2007 IEEE International Symposium on Approximate Dynamic Programming and Reinforcement Learning}, 
  title={Computing Optimal Stationary Policies for Multi-Objective Markov Decision Processes}, 
  year={2007},
  volume={},
  number={},
  pages={158-165},
  doi={10.1109/ADPRL.2007.368183}}

@inproceedings{inproceedings_aiss,
author = {Aissani, Nassima and Beldjilali, B and Trentesaux, Damien},
year = {2008},
month = {01},
pages = {},
title = {EFFICIENT AND EFFECTIVE REACTIVE SCHEDULING FOR MANUFACTURING SYSTEMS USING SARSA MULTI-OBJECTIVE AGENTS}
}

@phdthesis{Shabani_2009, series={Electronic Theses and Dissertations (ETDs) 2008+}, title={Incorporating flood control rule curves of the Columbia River hydroelectric system in a multireservoir reinforcement learning optimization model}, url={https://open.library.ubc.ca/collections/ubctheses/24/items/1.0063141}, DOI={http://dx.doi.org/10.14288/1.0063141}, school={University of British Columbia}, author={Shabani, Nazanin}, year={2009}, collection={Electronic Theses and Dissertations (ETDs) 2008+}}

@article{10.4018/jats.2009040104,
author = {Guo, Ying and Zeman, Astrid and Li, Rongxin},
title = {A Reinforcement Learning Approach to Setting Multi-Objective Goals for Energy Demand Management},
year = {2009},
issue_date = {April 2009},
publisher = {IGI Global},
address = {USA},
volume = {1},
number = {2},
issn = {1943-0744},
url = {https://doi.org/10.4018/jats.2009040104},
doi = {10.4018/jats.2009040104},
journal = {Int. J. Agent Technol. Syst.},
month = apr,
pages = {55–70},
numpages = {16}
}

@inproceedings{inproceedings_Castelletti,
author = {Castelletti, Andrea and Pianosi, Francesca and Restelli, Marcello},
year = {2012},
month = {06},
pages = {1-8},
title = {Tree-based Fitted Q-iteration for Multi-Objective Markov Decision problems},
isbn = {978-1-4673-1488-6},
journal = {Proceedings of the International Joint Conference on Neural Networks},
doi = {10.1109/IJCNN.2012.6252759}
}

@inproceedings{10.1145/1555301.1555311,
author = {Perez, Julien and Germain-Renaud, C\'{e}cile and K\'{e}gl, Bal\'{a}zs and Loomis, Charles},
title = {Responsive elastic computing},
year = {2009},
isbn = {9781605585789},
publisher = {Association for Computing Machinery},
address = {New York, NY, USA},
url = {https://doi.org/10.1145/1555301.1555311},
doi = {10.1145/1555301.1555311},
booktitle = {Proceedings of the 6th International Conference Industry Session on Grids Meets Autonomic Computing},
pages = {55–64},
numpages = {10},
location = {Barcelona, Spain},
series = {GMAC '09}
}

@inproceedings{10.1145/1390156.1390162,
author = {Barrett, Leon and Narayanan, Srini},
title = {Learning all optimal policies with multiple criteria},
year = {2008},
isbn = {9781605582054},
publisher = {Association for Computing Machinery},
address = {New York, NY, USA},
url = {https://doi.org/10.1145/1390156.1390162},
doi = {10.1145/1390156.1390162},
booktitle = {Proceedings of the 25th International Conference on Machine Learning},
pages = {41–47},
numpages = {7},
location = {Helsinki, Finland},
series = {ICML '08}
}

@inbook{10.5555/3454287.3455598,
author = {Yang, Runzhe and Sun, Xingyuan and Narasimhan, Karthik},
title = {A generalized algorithm for multi-objective reinforcement learning and policy adaptation},
year = {2019},
publisher = {Curran Associates Inc.},
address = {Red Hook, NY, USA},
booktitle = {Proceedings of the 33rd International Conference on Neural Information Processing Systems},
articleno = {1311},
numpages = {12}
}

@article{ZHANG2023526,
	author = {Linzi Zhang and Zhiquan Qi and Yong Shi},
	doi = {https://doi.org/10.1016/j.procs.2023.08.018},
	issn = {1877-0509},
	journal = {Procedia Computer Science},
	note = {Tenth International Conference on Information Technology and Quantitative Management (ITQM 2023)},
	pages = {526-532},
	title = {Multi-objective Reinforcement Learning -- Concept, Approaches and Applications},
	url = {https://www.sciencedirect.com/science/article/pii/S1877050923007767},
	volume = {221},
	year = {2023}}

@article{survey1,
	author = {Hayes, Conor F. and R{\u a}dulescu, Roxana and Bargiacchi, Eugenio and K{\"a}llstr{\"o}m, Johan and Macfarlane, Matthew and Reymond, Mathieu and Verstraeten, Timothy and Zintgraf, Luisa M. and Dazeley, Richard and Heintz, Fredrik and Howley, Enda and Irissappane, Athirai A. and Mannion, Patrick and Now{\'e}, Ann and Ramos, Gabriel and Restelli, Marcello and Vamplew, Peter and Roijers, Diederik M.},
	date = {2022/04/13},
	doi = {10.1007/s10458-022-09552-y},
	id = {Hayes2022},
	isbn = {1573-7454},
	journal = {Autonomous Agents and Multi-Agent Systems},
	number = {1},
	pages = {26},
	title = {A practical guide to multi-objective reinforcement learning and planning},
	url = {https://doi.org/10.1007/s10458-022-09552-y},
	volume = {36},
	year = {2022}}

@INPROCEEDINGS{6889637,
  author={Van Moffaert, Kristof and Brys, Tim and Chandra, Arjun and Esterle, Lukas and Lewis, Peter R. and Nowé, Ann},
  booktitle={2014 International Joint Conference on Neural Networks (IJCNN)}, 
  title={A novel adaptive weight selection algorithm for multi-objective multi-agent reinforcement learning}, 
  year={2014},
  volume={},
  number={},
  pages={2306-2314},
  doi={10.1109/IJCNN.2014.6889637}}

@article{10.5555/2591248.2591251,
author = {Roijers, Diederik M. and Vamplew, Peter and Whiteson, Shimon and Dazeley, Richard},
title = {A survey of multi-objective sequential decision-making},
year = {2013},
issue_date = {October 2013},
publisher = {AI Access Foundation},
address = {El Segundo, CA, USA},
volume = {48},
number = {1},
issn = {1076-9757},
journal = {J. Artif. Int. Res.},
month = oct,
pages = {67–113},
numpages = {47}
}

@INPROCEEDINGS{5396122,
  author={Soorapanth, Theerachet},
  booktitle={TENCON 2009 - 2009 IEEE Region 10 Conference}, 
  title={Multi-objective circuit design with weight-factor optimization via geometric programming}, 
  year={2009},
  volume={},
  number={},
  pages={1-5},
  doi={10.1109/TENCON.2009.5396122}}

@article{VISAN2022109987,
	author = {C{\u a}t{\u a}lin Vi{\c s}an and Octavian Pascu and Marius St{\u a}nescu and Elena-Diana {\c S}andru and Cristian Diaconu and Andi Buzo and Georg Pelz and Horia Cucu},
	doi = {https://doi.org/10.1016/j.knosys.2022.109987},
	issn = {0950-7051},
	journal = {Knowledge-Based Systems},
	pages = {109987},
	title = {Automated circuit sizing with multi-objective optimization based on differential evolution and Bayesian inference},
	url = {https://www.sciencedirect.com/science/article/pii/S0950705122010802},
	volume = {258},
	year = {2022}}

@INPROCEEDINGS{11092100,
  author={Taşkıran, Hakan and Sağlıcan, Enes and Afacan, Engin},
  booktitle={2025 21st International Conference on Synthesis, Modeling, Analysis and Simulation Methods, and Applications to Circuits Design (SMACD)}, 
  title={Multi-Objective Optimization of Analog Circuits Using Reinforcement Learning}, 
  year={2025},
  volume={},
  number={},
  pages={1-4},
  doi={10.1109/SMACD65553.2025.11092100}}

@article{ruiz2017temporal,
	author = {Ruiz-Montiel, Manuela and Mandow, Lawrence and P{\'e}rez-de-la-Cruz, Jos{\'e}-Luis},
	journal = {Neurocomputing},
	pages = {15--25},
	publisher = {Elsevier},
	title = {A temporal difference method for multi-objective reinforcement learning},
	volume = {263},
	year = {2017}}

@article{qin2020energy,
	author = {Qin, Yao and Wang, Hua and Yi, Shanwen and Li, Xiaole and Zhai, Linbo},
	journal = {The Journal of Supercomputing},
	number = {1},
	pages = {455--480},
	publisher = {Springer},
	title = {An energy-aware scheduling algorithm for budget-constrained scientific workflows based on multi-objective reinforcement learning},
	volume = {76},
	year = {2020}}

@article{hu2020dynamic,
	author = {Hu, Xin and Zhang, Yuchen and Liao, Xianglai and Liu, Zhijun and Wang, Weidong and Ghannouchi, Fadhel M},
	journal = {IEEE Transactions on Broadcasting},
	publisher = {IEEE},
	title = {Dynamic Beam Hopping Method Based on Multi-Objective Deep Reinforcement Learning for Next Generation Satellite Broadband Systems},
	year = {2020}}

@article{huang2019learning,
	author = {Huang, Sandy H and Zambelli, Martina and Kay, Jackie and Martins, Murilo F and Tassa, Yuval and Pilarski, Patrick M and Hadsell, Raia},
	journal = {arXiv preprint arXiv:1903.08542},
	title = {Learning gentle object manipulation with curiosity-driven deep reinforcement learning},
	year = {2019}}

@article{lacerda2017multi,
	author = {Lacerda, Anisio},
	journal = {Neurocomputing},
	pages = {12--24},
	publisher = {Elsevier},
	title = {Multi-objective ranked bandits for recommender systems},
	volume = {246},
	year = {2017}}

@inproceedings{da2019multi,
	author = {da Silva Veith, Alexandre and de Souza, Felipe Rodrigo and de Assun{\c{c}}{\~a}o, Marcos Dias and Lef{\`e}vre, Laurent and dos Anjos, Julio Cesar Santos},
	booktitle = {Proceedings of the 48th International Conference on Parallel Processing},
	pages = {1--10},
	title = {Multi-Objective Reinforcement Learning for Reconfiguring Data Stream Analytics on Edge Computing},
	year = {2019}}

@article{van2014multi,
	author = {Van Moffaert, Kristof and Now{\'e}, Ann},
	journal = {The Journal of Machine Learning Research},
	number = {1},
	pages = {3483--3512},
	publisher = {JMLR.org},
	title = {Multi-objective reinforcement learning using sets of pareto dominating policies},
	volume = {15},
	year = {2014}}

@article{zhou2019optimization,
	author = {Zhou, Zhenpeng and Kearnes, Steven and Li, Li and Zare, Richard N and Riley, Patrick},
	journal = {Scientific reports},
	number = {1},
	pages = {1--10},
	publisher = {Nature Publishing Group},
	title = {Optimization of molecules via deep reinforcement learning},
	volume = {9},
	year = {2019}}

@book{sutton2018reinforcement,
	author = {Sutton, Richard S and Barto, Andrew G},
	publisher = {MIT press},
	title = {Reinforcement learning: An introduction},
	year = {2018}}

@inproceedings{10.1145/3649329.3657318,
author = {Gu, Tianchen and Lyu, Ruiyu and Bi, Zhaori and Yan, Changhao and Yang, Fan and Zhou, Dian and Cui, Tao and Liu, Xin and Zhang, Zaikun and Zeng, Xuan},
title = {HiMOSS: A Novel High-dimensional Multi-objective Optimization Method via Adaptive Gradient-Based Subspace Sampling for Analog Circuit Sizing},
year = {2024},
isbn = {9798400706011},
publisher = {Association for Computing Machinery},
address = {New York, NY, USA},
url = {https://doi.org/10.1145/3649329.3657318},
doi = {10.1145/3649329.3657318},
booktitle = {Proceedings of the 61st ACM/IEEE Design Automation Conference},
articleno = {233},
numpages = {6},
location = {San Francisco, CA, USA},
series = {DAC '24}
}

@article{10.1115/1.4069046,
    author = {De Santanna, Lorenzo and Guidotti, Giacomo and Mastinu, Gianpiero and Gobbi, Massimiliano},
    title = {Multi-Objective Optimal Design Based on Reinforcement Learning},
    journal = {Journal of Mechanical Design},
    volume = {147},
    number = {10},
    pages = {101703},
    year = {2025},
    month = {08},
    issn = {1050-0472},
    doi = {10.1115/1.4069046},
    url = {https://doi.org/10.1115/1.4069046},
    eprint = {https://asmedigitalcollection.asme.org/mechanicaldesign/article-pdf/147/10/101703/7517513/md-24-1765.pdf},
}

@inproceedings{10.1007/978-3-031-83191-1_17,
	address = {Cham},
	author = {Nair, Siddharth H. and Vallon, Charlott and Borrelli, Francesco},
	booktitle = {Systems Theory in Data and Optimization},
	editor = {Berberich, Julian and Iannelli, Andrea and Allg{\"o}wer, Frank},
	isbn = {978-3-031-83191-1},
	pages = {261--276},
	publisher = {Springer Nature Switzerland},
	title = {Multi-Objective Learning Model Predictive Control},
	year = {2025}}

@inproceedings{10.1007/978-3-642-87563-2_5,
	address = {Berlin, Heidelberg},
	author = {Bowman, V. Joseph},
	booktitle = {Multiple Criteria Decision Making},
	editor = {Thiriez, Herv{\'e} and Zionts, Stanley},
	isbn = {978-3-642-87563-2},
	pages = {76--86},
	publisher = {Springer Berlin Heidelberg},
	title = {On the Relationship of the Tchebycheff Norm and the Efficient Frontier of Multiple-Criteria Objectives},
	year = {1976}}

@INPROCEEDINGS{4631208,
  author={Voss, Thomas and Beume, Nicola and Rudolph, Gunter and Igel, Christian},
  booktitle={2008 IEEE Congress on Evolutionary Computation (IEEE World Congress on Computational Intelligence)}, 
  title={Scalarization versus indicator-based selection in multi-objective CMA evolution strategies}, 
  year={2008},
  volume={},
  number={},
  pages={3036-3043},
  doi={10.1109/CEC.2008.4631208}}

@article{deep-sea-benchmark,
	author = {Vamplew, Peter and Dazeley, Richard and Berry, Adam and Issabekov, Rustam and Dekker, Evan},
	date = {2011/07/01},
	doi = {10.1007/s10994-010-5232-5},
	id = {Vamplew2011},
	isbn = {1573-0565},
	journal = {Machine Learning},
	number = {1},
	pages = {51--80},
	title = {Empirical evaluation methods for multiobjective reinforcement learning algorithms},
	url = {https://doi.org/10.1007/s10994-010-5232-5},
	volume = {84},
	year = {2011}}

@inproceedings{felten_toolkit_2023,
	author = {Felten, Florian and Alegre, Lucas N. and Now{\'e}, Ann and Bazzan, Ana L. C. and Talbi, El Ghazali and Danoy, Gr{\'e}goire and Silva, Bruno Castro da},
	title = {A Toolkit for Reliable Benchmarking and Research in Multi-Objective Reinforcement Learning},
	booktitle = {Proceedings of the 37th Conference on Neural Information Processing Systems ({NeurIPS} 2023)},
	year = {2023}
}
\bibliographystyle{icml2026}

\newpage
\appendix
\onecolumn

\section{Principle of Aligned Dominance}

\begin{lemma}[The Principle of Aligned Dominance]\label{lem:aligned_dominance}
    Let $\mathcal{V} \subset \mathbb{R}^d_{\geq 0}$ be a compact, non-empty set, such that $\mathcal{V} \neq \{\mathbf{0}\}$, and let $v \in \mathcal{V}$ be a non-zero vector. Consider any maximizer $v^*$ of the Chebyshev scalarization with the preference aligned with $v$. Then $v^*$ weakly dominates $v$.
    
    Formally:
    \begin{equation}
        \text{if } w = \frac{v}{\|v\|_2} \quad \text{and} \quad v^* \in \operatorname*{argmax}_{z \in \mathcal{V}} O_{w}(z), \quad \text{then } v \preceq v^*.
    \end{equation}
\end{lemma}

\begin{proof}
    If $v^* \in \operatorname*{arg\,max}\limits_{z \in \mathcal{V}} O_{w}(z)$ then $O_w(v^*)\geq O_w(v)$. Therefore $\min\limits_{i:v_i>0}\Bigl(\frac{v^*_i}{\frac{v_i}{\|v\|_2}}\Bigr)\geq \min\limits_{i:v_i>0}\Bigl(\frac{\cancel{v_i}}{\frac{\cancel{v_i}}{\|v\|_2}}\Bigr) =\|v\|_2$. Therefore for every $i=1,2\ldots d$, $v^*_i\geq \frac{v_i}{\cancel{\|v\|_2}}\cancel{\|v\|_2}\geq 0$. Guaranteeing that $v^*$ dominates $v$, i.e. $v^*\succeq v$.
\end{proof}

Lemma~\ref{lem:aligned_dominance} establishes a mechanism for monotonic improvement in the vector space. 
Specifically, it guarantees that for any point $v \in \V$, the scalarization objective $O_w$ parameterized by the direction $w= \frac{v}{\|v\|_2}$ aligned with $v$ is maximized by a point $v^*$ that weakly dominates $v$.
We note that there could be multiple maximizers, possibly incomparable to each other. However, every maximizer would dominate $v$. Example~\ref{app:eg:incomparable_aligned_dominance} illustrates this.
\begin{example}\label{app:eg:incomparable_aligned_dominance}
    Consider $v_1:=(1,1)$, $v_2=(1.5,2)$, $v_3=(2,1.5)$, and $\V:=\{v_1, v_2, v_3\}$ be three vectors. 
    Then, both $v_2$ and $v_3$ are maximizers of the Chebyshev scalarization aligned along $v_1$, however both $v_2$ and $v_3$ are incomparable.
\end{example}

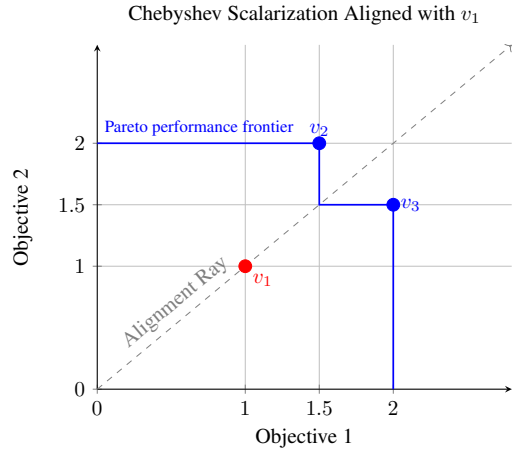
\begin{figure}[H] \centering \scalebox{0.8}{    
    \begin{tikzpicture} \begin{axis}[ axis lines = left, xlabel = {Objective 1}, ylabel = {Objective 2}, xmin=0, xmax=2.8, ymin=0, ymax=2.8, xtick={0,1,1.5,2}, ytick={0,1,1.5,2}, grid=major, legend pos=south east, title={Chebyshev Scalarization Aligned with $v_1$} ]
\coordinate (O) at (axis cs:0,0);
\coordinate (v1) at (axis cs:1,1);
\coordinate (v2) at (axis cs:1.5,2);
\coordinate (v3) at (axis cs:2,1.5);
\coordinate (corner) at (axis cs:1.5,1.5);

\addplot[dashed, gray, domain=0:2.8, samples=2, ->] {x} node[pos=0.35, above left, sloped] {Alignment Ray};

\addplot[only marks, mark=*, mark size=3pt, color=red] coordinates {(1,1)};
\node[below right, color=red] at (v1) {$v_1$};

\addplot[only marks, mark=*, mark size=3pt, color=blue] coordinates {(1.5,2) (2,1.5)};
\node[above, color=blue] at (v2) {$v_2$};
\node[right, color=blue] at (v3) {$v_3$};




\draw[blue, thick] (axis cs:0, 2) -- (axis cs:1.5, 2) -- (axis cs:1.5, 1.5) -- (axis cs:2, 1.5) -- (axis cs:2, 0) ;
\node[blue, above right, font=\fontsize{8}{8}\selectfont, align=right] at (axis cs: 0,2) {Pareto performance frontier};


\end{axis}
\end{tikzpicture}
}
\caption{Visualizing Example \ref{app:eg:incomparable_aligned_dominance}. The preference direction $w$ is aligned with the  point $v_1 = (1,1)$. The Chebyshev scalarization is maximized  simultaneously by points $v_2$ and $v_3$, even though they are incomparable with each other.}
\label{fig:chebyshev_example}

\end{figure}

\section{Chebyshev Scalarization and Pareto-Optimality}
\label{app:chebshev_scalarization}

\begin{theorem}[Sufficiency and Necessity of Chebyshev Scalarization]\label{app:thm:cheb_suff_nece}
Let $\V \subset \mathbb{R}_{\geq 0}^d$ 
be a set of achievable value vectors such that $\V\subseteq [0,{R_{\max}}]^d$, and let the $\ppf$ denote the set of Pareto performance frontier for vectors in $\V$.
For a given preference $w\in \W$, let $O_w(v)$ denote the weighted Chebyshev scalarization objective. The following properties hold:
\begin{enumerate}[label=\thetheorem.\arabic*]
    \item \label{app:thm:cheb_nece} \textbf{Sufficiency:} For any preference $w$, the maximizer of the Chebyshev scalarization is on the Pareto performance frontier. Formally, if $v^* \in \arg\max\limits_{v \in \mathcal{V}} O_w(v)$, then $v^* \in \ppf$.
    \item \label{app:thm:cheb_suff}  \textbf{Necessity:}  For any Pareto-optimal point $v^* \in \mathcal{V}$, there exists a preference $w \in \mathcal{W}$ such that $v^*$ is the unique maximizer of the scalarization. Formally if ${v}' \in \arg\max\limits_{v \in \mathcal{V}} O_w(v)$ then $v'=v^*$.
\end{enumerate}
\end{theorem}
\begin{proof}
\textbf{1. Sufficiency}

\textit{If $v^*\in \arg\max\limits_{v\in \V} O_w(v)$ then $v^*\in\ppf$}:
Suppose for the sake of contradiction $v^*$ is a maximizer of $O_w(v)$ but $v^*\notin \ppf$. 
Since $v^*$ is not on the $\ppf$ there is a Pareto-optimal vector $v'\in \V$ such that $v'_i > v^*_i$ for all $i = 1,2 \dots, d$. 
Then for every $i\in \{1,2,\ldots d\}$ such that $w_i>0$, we can divide $v_i'$ and $v_i^*$ by the positive weight $w_i$ to obtain:
\begin{equation*}
    \frac{v'_i}{w_i} > \frac{v^*_i}{w_i}
\end{equation*}
Taking the minimum over $i$:
\begin{equation*}
    \min_{i:w_i>0} \left( \frac{v'_i}{w_i} \right) > \min_{i:w_i>0} \left( \frac{v^*_i}{w_i} \right) \implies O_w( v') > O_w( v^*)
\end{equation*}
This contradicts the optimality of $v^*$. Thus, $v^*$ must be in $\ppf$.

\bigskip

\textbf{2. Necessity}

Let $v^*\in \mathcal{V}$ be a Pareto-optimal point, and let $w=\frac{v^*}{\|v^*\|_2}$ be the preference aligned with $v^*$. 
We invoke Lemma~\ref{lem:aligned_dominance} to obtain that for every $v'\in \arg\max\limits_{v\in \mathcal{V}} O_w(v)$, we have $v'\succeq v^*$. 
However, since $v^*$ is Pareto-optimal, this condition implies $v^*=v'$.
Thus, there is a unique maximizer along the preference aligned with $v^*$.
\end{proof}

\begin{theorem}[Coverage and Exact Characterization of Pareto Optimality]
\label{app:thm:pareto_coverage_characterization}
Let $\mathcal{V} \subset \mathbb{R}_{\geq 0}^d$ be a compact set of achievable value vectors, with $\mathcal{V}^{PO}$ denoting the Pareto-optimal set and $\ppf$ the Pareto performance frontier.

The relationship between Chebyshev scalarization and the Pareto set is characterized by the following two properties:

\begin{enumerate}[label=(\roman*), leftmargin=*]
    \item \textbf{Scalarization-Induced Coverage:} \label{app:item:coverage}
    Let $\mathcal{S}$ be the set constructed by selecting, for every preference $w \in \mathcal{W}$, an arbitrary vector that maximizes the scalarization $O_w$. Formally:
    \[
    \mathcal{S} \;:=\; \bigcup_{w \in \mathcal{W}} \big\{ v_w \big\}, \quad \text{where } v_w \in \arg\max_{v \in \mathcal{V}} O_w(v).
    \]
    Then $\mathcal{S}$ contains all Pareto-optimal points and consists solely of points on the Pareto performance frontier:
    \[ \mathcal{V}^{PO} \subseteq \mathcal{S} \subseteq \ppf. \]

    \item \textbf{Exact Characterization via Norm-Regularization:} \label{app:item:exactness}
    Let $\mathcal{M}$ be the set constructed by selecting, for every preference $w \in \mathcal{W}$, a vector $v^*_w$ obtained via a two-stage optimization process: first maximizing the Chebyshev scalarization $O_w$, and then maximizing the $\ell_2$-norm among the resulting candidates for tie-breaking. Formally:
    \[
    \mathcal{M} \;:=\; \bigcup_{w \in \mathcal{W}} \big\{ v^*_w \big\},
    \]
    where $v^*_w$ is an arbitrary solution to the nested optimization:
    \[
    v^*_w \in \arg\max_{v} \left\{ \|v\|_2 \;\Big|\; v \in \arg\max_{u \in \mathcal{V}} O_w(u) \right\}.
    \]
    Then $\mathcal{M}$ coincides with the Pareto-optimal set:
    \[ \mathcal{M} = \mathcal{V}^{PO}. \]
\end{enumerate}
\end{theorem}
\begin{proof}
\begin{enumerate}[font=\bfseries,label=\roman*.]
    \item \textbf{Scalarization-Induced Coverage:} The proof proceeds in two steps:

1. \textbf{($\mathcal{S} \subseteq \ppf$):}
Consider any selected element $v_w \in \mathcal{S}$. By definition, $v_w$ is a maximizer of the weighted Chebyshev scalarization $O_w$. From the Sufficiency property (Theorem~\ref{thm:cheb_nece}), any such maximizer must lie on the Pareto performance frontier. Therefore, every element in $\mathcal{S}$ belongs to $\ppf$, implying $\mathcal{S} \subseteq \ppf$.

2. \textbf{($\mathcal{V}^{PO} \subseteq \mathcal{S}$):}
Let $v^* \in \mathcal{V}^{PO}$ be an arbitrary Pareto-optimal vector. 
From the Necessity property (Theorem~\ref{thm:cheb_suff}), there exists a specific preference $w \in \mathcal{W}$ such that $v^*$ is the \textit{unique} maximizer of the scalarization $O_w$ over $\mathcal{V}$.
Since the maximizer is unique, the arbitrary selection $v_w$ for this specific weight must be $v^*$ itself (i.e., $v_w = v^*$). Consequently, $v^*$ is included in $\mathcal{S}$, implying $\mathcal{V}^{PO} \subseteq \mathcal{S}$.
\item \textbf{Exact Characterization via Norm-Regularization:} We prove the equality by showing mutual inclusion: $\mathcal{M} \subseteq \mathcal{V}^{PO}$ and $\mathcal{V}^{PO} \subseteq \mathcal{M}$.

\textbf{1. ($\mathcal{M} \subseteq \mathcal{V}^{PO}$):}
Let $v_w \in \mathcal{M}$ be a vector selected for some preference $w$. Suppose for the sake of contradiction that $v_w \notin \mathcal{V}^{PO}$.
This implies there exists a dominating vector $v^* \in \mathcal{V}^{PO}$ such that ${v_w} \prec v^*$.
This means $v_{w,k} \leq v^*_k$ for all components $k$, with strict inequality for at least one component.

This dominance implies two properties:
\begin{enumerate}
    \item \textbf{Scalarization:} Since every component of $v^*$ is at least as large as $v_w$, the weighted minimum ratio cannot decrease, so $O_w(v_w) \leq O_w(v^*)$.
    \item \textbf{Norm:} Since $v^*$ strictly dominates $v_w$ and values are non-negative, the $\ell_2$-norm must be strictly larger: $\|v_w\|_2 < \|v^*\|_2$.
\end{enumerate}

Now we check if $v_w$ could have been the output of the two-stage process:
\begin{itemize}
    \item If $O_w(v^*) > O_w(v_w)$, then $v_w$ failed the first stage (maximizing scalarization), which is a contradiction.
    \item If $O_w(v^*) = O_w(v_w)$, then $v_w$ is a valid candidate for the second stage. However, since $\|v^*\|_2 > \|v_w\|_2$, the second stage would have selected $v^*$ rather than $v_w$. 
    This is also a contradiction.
\end{itemize}
Thus, no such dominating vector $v^*$ can exist, so $v_w$ must be Pareto-optimal.

\textbf{2.  ($\mathcal{V}^{PO} \subseteq \mathcal{M}$):}
Let $v^* \in \mathcal{V}^{PO}$ be an arbitrary Pareto-optimal point. Consider the specific preference aligned with this vector: $w^* := \frac{v^*}{\|v^*\|_2}$.

We invoke Lemma~\ref{lem:aligned_dominance}, which states that for aligned preferences, any maximizer $v'\in \arg\max\limits_{v\in\V}O_{w^*}(v)$ of the Chebyshev scalarization $O_{w^*}$ must weakly dominate the alignment target: $v' \succeq v^*$.
However, since $v^*$ is Pareto-optimal, it cannot be strictly dominated. Therefore, the weak dominance implies equality: $v' = v^*$.

This means the set of first-stage maximizers $\arg\max\limits_{u \in \mathcal{V}} O_{w^*}(u)$ is the singleton set $\{v^*\}$. Since the set contains only one element, the second-stage norm maximization trivially selects $v^*$. Consequently, $v^* \in \mathcal{M}$.
\end{enumerate}
\end{proof}

\section{Bellman Update}

\begin{theorem}[Recursive Decomposition of $Q^*$]\label{app:thm:rec_decomp}
Let $\mathbf{Q}^*: S \times A \times \mathcal{W} \to \mathbb{R}^d_{\geq 0}$ be the optimal preference action-value function. For any state $s \in S$, action $a \in A$, and preference vector $w \in \mathcal{W}$, $\mathbf{Q}^*$ satisfies the recursive relationship:

\begin{equation}\label{app:thm:eq:rec_q_star}
\mathbf{Q}^*(s, a, w) = R(s, a) + \gamma \E\limits_{s' \sim p(\cdot \mid s, a)} \left[ \mathbf{Q}^*(s', \hat{A}(s'), \hat{\mathcal{W}}(s')) \right]
\end{equation}
where the functions $\hat{A}: S \to A$ and $\hat{\mathcal{W}}: S \to \mathcal{W}$ map the next-states to actions and preferences, respectively.
\end{theorem}

\begin{proof}
The proof proceeds by expanding the definition of the optimal preference action-value function and examining the optimality of the induced ``tail" policies.

\textbf{1. Expansion of $Q^*$:}
By Definition~\ref{def:optimal_pref_act_val_fun}, $Q^*(s,a,w)$ is the value achieved by a policy $\pi$ that maximizes the Chebyshev scalarization for the preference $w$. Let us denote this optimal policy by $\pi_w$. We can expand the value of this policy recursively:
\begin{equation}
\label{eq:proof:rec_expansion}
Q^*(s,a,w) = R(s,a) + \gamma \E_{s' \sim p(\cdot \mid s,a)} \left[ R(s', a') + \gamma V^{\pi_w(\cdot |s',a')}(p(\cdot \mid s', a')) \right]
\end{equation}
where $a' = \pi_w(s')$ is the action chosen by the policy at the next state, and $V^{\pi_w(\cdot |s',a')}(p(\cdot \mid s', a'))$ denotes the expected value of following policy $\pi_w$ from that point onward (i.e., given the history $(s',a')$).

\textbf{2. The Claim:}
To establish the recursive relationship in Equation~\eqref{app:thm:eq:rec_q_star}, we must prove that the tail value term corresponds to an optimal value $Q^*$ for some configuration of action and preference. Specifically, we claim that for every next state $s'$, the continuation value vector $V^{\pi_w(\cdot\mid s',a')}(p(\cdot \mid s', a'))$ is Pareto-optimal. If it is Pareto-optimal, then by  Theorem (Theorem~\ref{thm:cheb_suff}), it must be the maximizer for some specific preference $w'$.

\textbf{3. Proof by Contradiction (The ``Stitching" Argument):}
Suppose, for the sake of contradiction, that for the next-state distribution induced by a specific state $s'$ and action $a'=\pi_w(s')$, the continuation policy induced by $\pi_w(\cdot\mid s',a')$ is \textit{not} Pareto-optimal. 
This implies there exists an alternative policy $\pi_{\text{better}}$ that strictly dominates the original ``tail" policy from the distribution $p(\cdot \mid s', a')$. Formally:
\[
V^{\pi_{\text{better}}}(p(\cdot \mid s', a')) \succ V^{\pi_w(\cdot\mid s',a')}(p(\cdot \mid s', a'))
\]
We now construct a ``stitched" composite policy $\pi_{\text{new}}$ that switches behavior based on the history. The policy follows $\pi_{\text{better}}$ if the trajectory passes through the specific history $(s', a')$, and defaults to $\pi_w$ otherwise:
\begin{equation}
\pi_{\text{new}}(h) := 
\begin{cases} 
\pi_{\text{better}}(h) & \text{if } h \text{ is rooted at } (s', a') \\ 
\pi_w(h) & \text{otherwise} 
\end{cases}
\end{equation}
Since $\pi_{\text{better}}$ provides strictly higher value on the specific sub-branch starting at $(s', a')$ and $\pi_{\text{new}}$ behaves identically to $\pi_w$ everywhere else, the multi-objective value at the root $(s,a)$ must strictly dominate, i.e. $V^{\pi_w}(p(\cdot\mid s,a))\prec V^{\pi_{\text{better}}}(p(\cdot\mid s,a))$. 

Substituting this back into Equation~\eqref{eq:proof:rec_expansion}, the value of the new policy satisfies:

\begin{equation}
\resizebox{\columnwidth}{!}{%
$
R(s,a) + \gamma \mathbb{E} \Big[ R(s',\pi_w(s'))+\gamma V^{\pi_{\text{better}}(\cdot\mid s',\pi_w(s'))}\big(p(\cdot \mid s',\pi_w(s'))\big) \Big] \succ R(s,a) + \gamma \mathbb{E} \Big[ R(s',\pi_w(s'))+\gamma V^{\pi_w(\cdot\mid s',\pi_w(s'))}\big(p(\cdot \mid s',\pi_w(s'))\big) \Big] = Q^*(s,a,w)
$
}
\end{equation}

This implies that $\pi_{\text{new}}$ achieves a strictly dominating value vector than $\pi_w$.
Consequently, from Lemma~\ref{item:exactness} the two-stage optimization process involving the Chebyshev scalarization and tie-breaking would never choose $\pi_w$ since $\pi_{\text{better}}$ strictly dominates the value $\pi_w$ obtains from the intial distribution induced by $p(\cdot\mid s,a)$. 
That is $\mathbf{V}^{\pi_w}(p(\cdot \mid s,a)) \notin \arg\max\limits_{V} \Big\{ \|\mathbf{V}\|_2 \;\big|\; \mathbf{V} = \mathbf{V}^\pi(p(\cdot \mid s,a)),\pi \in \arg\max\limits_{\pi' \in \Pi} O_w(\mathbf{V}^{\pi'}(p(\cdot \mid s,a))) \Big\}$.

This contradicts the original assumption that $\pi_w$ was the maximizer for the preference $w$.

\textbf{4. Conclusion:}
Because the contradiction assumes the tail was not optimal, we conclude that the induced tail policy must be Pareto-optimal.
Since the continuation value is Pareto-optimal, there exists a preference $w'$ (specifically, the one aligned with the continuation value)  from Theorem~\ref{thm:cheb_suff} such that:
\[
R(s', a') + \gamma V^{\pi_w}(p(\cdot \mid s', a')) = Q^*(s', a', w')
\]
Substituting this back into Equation~\eqref{eq:proof:rec_expansion} yields the desired recursive relationship:
\[
Q^*(s, a, w) = R(s, a) + \gamma \E\limits_{s' \sim p(\cdot \mid s, a)} \left[ Q^*(s', \hat{A}(s'), \hat{\mathcal{W}}(s')) \right]
\]
\end{proof}

\section{Convergence Proof}\label{app:sec:conv_proof}

\textbf{Notation:} For brevity, we denote the estimate at the $ n$-th Bellman update $ \T^{ n}  Q$ by the shorthand $ \hat{ Q}_{ n}$.

\begin{theorem}\label{app:thm:bellman_env_conv}
    Let the value function be initialized to an optimistic upper bound $\hat{{Q}}_0(s,a,w) :=  \R=\frac{{R}_{\max}}{ 1-\gamma}\Biggl(\begin{smallmatrix}
1\\[-1pt]
1\\[-5pt]
\vdots\\[-0.5pt]
1
\end{smallmatrix}
\Biggr)\in \mathbb{R}^d_{\geq 0}$. Then for every $n \in \mathbb{N}$, state $s$, and action $a$, the following three conditions hold:
    \begin{enumerate}
        \item \textbf{Upper-bound Coverage:} The estimated set of values covers the true Pareto front.
        Formally, for every preference $w$, there exists a preference $w'$ such that:
        \[ {Q}^*(s,a,w) \preceq \hat{{Q}}_n(s,a,w') \]
        \item \textbf{Envelope Property:} The estimate is never goes below the Pareto performance frontier. 
        Formally, there is no preference $w'$ for which a Pareto-optimal value is strictly greater than the estimate in \textbf{all} components:
        \[ 
        \forall w, \nexists w' \quad \text{such that} \quad \hat{Q}_n(s,a,w)_i < Q^*(s,a,w')_i \quad \forall i=1,\dots,d 
        \]
                
        \item \textbf{Asymptotic Convergence:} Every estimate $\hat{{Q}}_n$ corresponds to a feasible policy execution up to a residual term.
        There exists a policy $\pi \in \Pi$ such that:
        \[ \hat{{Q}}_n(s,a,w) - {V}^\pi(s) \preceq \gamma^n {R} \]
    \end{enumerate}
\end{theorem}

\begin{proof}

    \begin{enumerate}
        \item \textbf{Upper-bound Coverage:} 
        We prove this by induction on $n$.
        
        \noindent \textbf{Base Case ($n=0$):}
        By initialization, $\hat{{Q}}_0(s,a,w) = \frac{\R}{1-\gamma}$, which is the maximum possible return. Thus, ${Q}^*(s,a,w) \preceq \hat{{Q}}_0(s,a,w)$ trivially holds.

        \noindent \textbf{Induction Step:}
        Assume the hypothesis holds for $n=k$. That is, for any $s', a', w_{next}$, there exists some $\tilde{w}$ such that $Q^*(s',a',w_{next}) \preceq \hat{Q}_k(s',a',\tilde{w})$.
        We wish to show that for any $w$, there exists a $w'$ such that $Q^*(s,a,w)\preceq \hat{Q}_{k+1}(s,a,w')$.
        
        From the recursive decomposition (Theorem \ref{thm:rec_decomp}), we know:
        \[ Q^*(s, a, w) = R(s, a) + \gamma \E_{s' \sim p(\cdot \mid s, a)} \left[ Q^*(s', a'', w'') \right] \]
        We construct the specific target preference $w'$ aligned with this optimal continuation expectation:
        \[ w' := \frac{\E_{s' \sim p(\cdot \mid s, a)} \left[ Q^*(s', a'', w'') \right]}{\big\|\E_{s' \sim p(\cdot \mid s, a)} \left[ Q^*(s', a'', w'') \right]\big\|_2} \]
        
        By the induction hypothesis, for every optimal continuation value inside the expectation, there exists a preference in our estimate set $\hat{Q}_k$ that dominates it. 
        Specifically, for any $s'$, the value $Q^*(s', a'', w'')$ is dominated by some $\hat{Q}_k(s', a''', w''')$.
        
        Because the Bellman operator computes the maximizer over all possible actions and preferences, and since the specific preference $w'$ is perfectly aligned with the optimal continuation vector $\E_{s' \sim p(\cdot \mid s, a)} \left[ Q^*(s', a'', w'') \right]$, Lemma~\ref{lem:aligned_dominance} (Principle of Aligned Dominance) guarantees:
        \[
        \max_{a, \tilde{w}} O_{w'}\Big( \E\big[ \hat{Q}_k(s', a, \tilde{w}) \big] \Big) \ge O_{w'}\Big( \E\big[ Q^*(s', a'', w'') \big] \Big)
        \]
        This implies component-wise dominance for the updated value:
        \[ \E_{s'}\big[ \hat{Q}_k(s', \hat{a}, \hat{w}) \big] \succeq \E_{s'}\big[ Q^*(s', a'', w'') \big] \]
        Multiplying by $\gamma$ and adding the reward $R(s,a)$ to both sides proves that $Q^*(s,a,w)\preceq \hat{Q}_{k+1}(s,a,w')$.

\item \textbf{Envelope Property:}

        We prove this by contradiction for $n\geq 1$, if $n=0$, then $\hat{Q}_0(s,a,w)=\R$ and the Envelope property holds.

        Assume there exist preferences $w, w'$ such that the estimate is strictly dominated in every component:
        \begin{equation}\label{eq:env_assumption}
            \hat{Q}_n(s,a,w)_i < Q^*(s,a,w')_i \quad \text{for every } i=1,\dots,d
        \end{equation}

        Let the update for $\hat{Q}_n(s,a,w)$ be derived from the specific choice of action $\hat{a}$ and preference $\hat{w}$. 
        Similarly, let the optimal value $Q^*(s,a,w')$ be derived from the optimal action $a^*$ and preference $w^*$.
        
        Expanding both sides of Eq.~\ref{eq:env_assumption}, subtracting $R(s,a)$, and dividing by $\gamma$ yields:
        \begin{equation}\label{eq:env_strict_ineq}
            \E_{s'}\big[ \hat{Q}_{n-1}(s',\hat{a},\hat{w})\big]_i < \E_{s'}\big[Q^*(s',a^*,w^*)\big]_i \quad \forall i
        \end{equation}

        We now apply the \textbf{Upper-bound Coverage} property. For the optimal pair $(a^*, w^*)$, there must exist a preference $\tilde{w}$ such that:
        \begin{equation}\label{eq:env_weak_ineq}
             Q^*(s',a^*,w^*) \preceq \hat{Q}_{n-1}(s',a^*,\tilde{w})
        \end{equation}
        
        Substituting (\ref{eq:env_weak_ineq}) into (\ref{eq:env_strict_ineq}) gives the chain of inequalities:
        \begin{equation}
            \E_{s'}\big[ \hat{Q}_{n-1}(s',\hat{a},\hat{w})\big]_i < \E_{s'}\big[Q^*(s',a^*,w^*)\big]_i \leq \E_{s'}\big[\hat{Q}_{n-1}(s',a^*,\tilde{w})\big]_i \quad \forall i
        \end{equation}

        This implies that the vector expected from the chosen pair $(\hat{a}, \hat{w})$ is  component-wise strictly greater than the vector from the pair $(a^*, \tilde{w})$.
        
        Since any valid scalarization function is strictly increasing with respect to strong dominance, the scalarized score for $(a^*, \tilde{w})$ must be strictly higher than the score for $(\hat{a}, \hat{w})$. 
        That is, for any preference $w$, the Chebyshev scalarization $O_{w''}(z) = \min\limits_{i:w''_i>0}\Bigl( \frac{z_i}{{w''}_i}\Big)$ satisfies strict monotonicity with respect to strong dominance. Specifically, since the vector for $(\hat{a}, \hat{w})$ is strictly smaller than the vector for $(a^*, \tilde{w})$ in every component $i$, it follows that:
        \[
            O_{w''}\Big(\E_{s'}\big[ \hat{Q}_{n-1}(s',\hat{a},\hat{w})\big]\Big) = \min\limits_{i:w''_i>0}\bigg(\frac{\E[\dots]_{i}}{w''_{i}}\bigg) < \min\limits_{i:w''_i>0}\bigg(\frac{\E[\dots]_{i}}{w''_{i}}\bigg) = O_{w''}\Big(\E_{s'}\big[ \hat{Q}_{n-1}(s',a^*,\tilde{w})\big]\Big)
        \]
        
        This strict inequality violates the optimality of the update rule, as the Bellman operator should have selected the superior configuration $(a^*, \tilde{w})$ (or an even better one) instead of $(\hat{a}, \hat{w})$.
        
        Thus the assumption is not true, and the Envelope property holds.

        \item \textbf{Asymptotic Convergence:} 
        Consider an estimate $\hat{{Q}}_n(s,a,w)$. We construct a non-stationary deterministic policy $\pi$ defined by the ``backtracking" trace of the Bellman updates.
        
        Let $\pi$ execute action $a$ at $t=0$. For steps $t=1 \dots n-1$, let $\pi$ select the actions that were chosen as the maximizers during the recursive computation of $\hat{{Q}}_n$. 
        The accumulated rewards of this policy match the estimate exactly for the first $n$ steps.
        
        The difference between the estimate and the true value of $\pi$ arises only from the tail (after step $n$). Since rewards are bounded, the maximum possible value of the tail is bounded by $\gamma^n\R$.
        Thus:
        \[ \bm 0 \preceq \hat{{Q}}_n(s,a,w) - {V}^\pi(s) \preceq \gamma^n \R \]
    \end{enumerate}
\end{proof}

\subsection{Approximate Pareto Coverage}\label{app:sec:approx_POC}

In this subsection, we provide a detailed algorithmic characterization of the \textbf{policy execution} (or ``unrolling'') procedure derived from our Bellman operator, and we prove its convergence properties.

We begin by establishing the necessary notation and definitions. The core intuition is that the computation of the $n$-th Bellman update $\mathcal{T}^n Q(s_1, a_1, w_1)$ naturally relies on the optimal substructure of $\mathcal{T}^{n-1} Q(s_2, a_2, w_2)$, which in turn relies on $\mathcal{T}^{n-2} Q$, and so on. This recursive dependency allows us to ``unroll'' the sequence of actions and preferences $(a_1, w_1), (a_2, w_2), \dots$ originally selected during the computation of $\mathcal{T}^n Q$. We formally define the policy that performs this unrolling in Definition~\ref{def:induced_policy}.

Since the value function is initialized with the maximum obtainable rewards, $\mathcal{R} = \frac{\mathbf{R}_{\max}}{1-\gamma} \mathbf{1}\in \mathbb{R}^d_{\geq 0}$, the Bellman estimate $\mathcal{T}^n Q(s,a,w)$ can be decomposed into the sum of rewards obtained by this unrolled policy over $n$ steps plus a \textbf{residual term} $\gamma^n \mathcal{R}$.

Crucially, because $\mathcal{T}^n Q$ is simply this finite-horizon sum shifted by the constant residual $\gamma^n \mathcal{R}$, optimizing the Chebyshev scalarization $O_w(\mathcal{T}^n Q)$ allows us to identify Pareto-optimal policies for the $n$-step horizon. In Theorem~\ref{app:thm:duality}, we establish a rigorous \textbf{one-to-one correspondence} between the set of Pareto-optimal policies and the Pareto-optimal $\mathcal{T}^n Q$ values for each state and action.

This correspondence serves as the theoretical foundation for our deployment procedure. Given an $n$-step Bellman estimate $\mathcal{T}^n Q$, our policy selects an action that maximizes this estimate. By greedily optimizing for the $n$-step horizon at each step, we ensure that the resulting policy \textbf{weakly dominates} the $(n-1)$-step policy that was implicitly used to compute the optimal continuation $\mathcal{T}^{n-1} Q$. This monotonic improvement property allows us to prove the \textbf{Simulation Lemma} (Theorem~\ref{app:thm:recursive_trajectory_error}), which confirms that the actual accumulated rewards of our greedy policy satisfy the bounds predicted by the Bellman estimate, thereby concluding the proofs for approximate coverage and parsimony.



\begin{definition}[Horizon-Dependent Greedy Selector Functions]
\label{def:greedy_selectors}
For any Bellman update iterate $k \ge 0$, current state $s \in S$, action $a \in A$, and preference $w \in \mathcal{W}$, let the selector functions
\[
\hat{A}_k(s, a, w, \cdot): S \to A \quad \text{and} \quad \hat{W}_k(s, a, w, \cdot): S \to \mathcal{W}
\]
be the mappings that maximize the scalarized expected value of the $(k)$-th estimate. They are defined as some solution to:
\[
\big(\hat{A}_k(s,a,w, \cdot), \hat{W}_k(s,a,w, \cdot)\big) \in \operatorname*{arg\,max}_{\substack{\phi_A: S \to A,\\ \phi_W: S \to \mathcal{W}}} O_w\Big( \mathbb{E}_{s' \sim p(\cdot|s,a)} \big[\mathcal{T}^{k}Q(s', \phi_A(s'), \phi_W(s'))\big] \Big)
\]
\end{definition}

Using the greedy selector functions, we are able to select an action $\hat{A}_i(s,a,w,s')$ and a preference $\hat{W}_i(s,a,w,s')$ at each step, allowing us to unroll the rewards $R(s', \hat{A}_i(s,a,w,s'))$ involved in computing $\mathcal{T}^n Q(s,a,w)$ up to $n$ steps.


\begin{definition}[Induced Policy Definition]
\label{def:induced_policy}
Let $n \in \mathbb{N}$ be the iterate index of the Bellman update $\mathcal{T}^n Q$. For a starting configuration of state $s \in S$, action $a \in A$, and preference $w \in \mathcal{W}$, we define the policy $\pi_{n}^{s,a,w}$ as a mapping from a history $h$ and a next state $s'$ to an action.

Let $H_{s,a}$ denote the set of all histories starting with the pair $(s,a)$. That is, any $h \in H_{s,a}$ of length $t$ is a sequence $h = (s_1, a_1, s_2, a_2, \dots, s_t, a_t)$ where $s_1 = s$ and $a_1 = a$.

To define the policy behavior, we first define the notion of  \textbf{history consistency} and \textbf{induced preference sequence}:

\begin{enumerate}[label=(\roman*)]
    \item \textbf{History Consistency:} A history $h$ is said to be \textit{consistent} if, for all steps $1 < i \le t$, the action taken matches the greedy selector for the induced preference:
    \[
    a_i = \hat{A}_{n-(i-1)}(s_{i-1}, a_{i-1}, w_{i-1}, s_i)
    \]

    \item \textbf{Induced Preference Sequence:} For a history $h$, we construct the sequence of target preferences $(w_1, \dots, w_t)$ recursively:
    \[
    w_i = 
    \begin{cases}
        w & \text{if } i = 1 \\
        \hat{W}_{n-(i-1)}(s_{i-1}, a_{i-1}, w_{i-1}, s_i) & \text{if } 1 < i \le t
    \end{cases}
    \]
    
\end{enumerate}

Based on these definitions, the policy $\pi^{s,a,w}_n(s' \mid h)$ is defined as follows:

\begin{enumerate}
    \item \textbf{Case $t \ge n$ (Horizon Exceeded):}
    If the history length exhausts the horizon ($t \ge n$), the policy behaves arbitrarily.
    
    \item \textbf{Case $t < n$ (Within Horizon):}
    \begin{itemize}
        \item If the history $h$ is \textbf{consistent}, the policy selects the greedy action for the next state $s'$ according to the current preference $w_t$:
        \[
        \pi^{s,a,w}_n(s' \mid h) := \hat{A}_{n-t}(s_t, a_t, w_t, s')
        \]
        \item If the history $h$ is \textbf{inconsistent}, the policy behaves arbitrarily.
    \end{itemize}
\end{enumerate}
\end{definition}

Having formally defined the induced policy $\pi^{s,a,w}_m$ which explicitly unrolls the recursive action selection inherent in the Bellman estimate $\mathcal{T}^n Q$, we now introduce the concept of the $n$-step expected return. 
This definition provides the  quantity needed to subsequently prove Lemma~\ref{app:lemm:bellman_consistency}, demonstrating that the value computed by the operator is realized by the execution of the induced policy upto a residual term $\gamma^n\R$.

\begin{definition}[$n$-Step Value]
Given a policy $\pi:(S\times A)^*\times S\rightarrow A$, the $n$-step expected return starting from state $s$ is defined as:
\[
V^\pi_n(s) := \mathbb{E}\left[ \sum_{t=1}^{n} \gamma^{t-1} R(s_t, a_t) \;\middle|\; s_1 = s, \, a_t = \pi(s_t \mid h_{t-1}) \right]
\]
where $h_{t-1} = (s_1, a_1, \dots, s_{t-1}, a_{t-1})$ denotes the history of state-action pairs prior to the current step (with $h_{0} = \emptyset$).
\end{definition}

\begin{lemma}[Bellman Consistency]
\label{app:lemm:bellman_consistency}
For any horizon $n \in \mathbb{N}$, state $s \in S$, action $a \in A$, and preference $w \in \mathcal{W}$, the induced policy $\pi = \pi^{s,a,w}_n$ satisfies:
\[
\mathcal{T}^n Q(s,a,w) - \gamma^n \mathcal{R} = V^{\pi}_n(s).
\]
\end{lemma}

\begin{proof}[Proof Sketch.]
The proof follows directly from the construction of the policy. The Bellman operator $\mathcal{T}^n$ is defined by a sequence of nested maximizations, where at each depth $t$, the operator selects the optimal strategy (action and preference) to maximize the continuation value.

By Definition~\ref{def:induced_policy}, the policy $\pi^{s,a,w}_n$ is explicitly constructed to replicate this optimization path. For every time step $t \le n$, the policy computes the sequence of greedy selector functions implied by the history and executes the exact action $\hat{a}_t$ that was chosen to maximize the corresponding stage of the Bellman update.

Since the policy chooses the action $a_t$ at every step to match the maximizer of the Bellman operator's $t$-th expansion, the distribution of trajectories induced by the policy is identical to those assumed in the calculation of $\mathcal{T}^nQ$. Consequently, the expected sum of discounted rewards generated by the policy (the value $V^\pi_n$) is identical to the value computed by the Bellman update (minus the discounted residual at step $n$).
\end{proof}

We next establish a one-to-one correspondence between the $n$-step values obtained by $n$-step Pareto-optimal policies and the $n$-step Bellman estimate in Theorem~\ref{app:thm:duality}.
However, prior to this, we prove in Lemma~\ref{lemma:subpolicy_optimality} that $n$-step Pareto-optimal policies exhibit optimal substructure, ensuring that the induced subpolicy following the first step is itself $(n-1)$-step Pareto-optimal.

\begin{lemma}[Pareto Optimal Subpolicy]
\label{lemma:subpolicy_optimality}
Let $\pi$ be a policy that is Pareto-optimal for the $n$-step value $V^\pi_n(s)$. Let $s'$ be any state reachable from $s$ at the second step with non-zero probability. 

Then, the \textbf{continuation policy} $\pi(\cdot \mid s, a)$, which represents the behavior of $\pi$ for all steps following the history $(s, a)$, must be Pareto-optimal for the remaining $(n-1)$-step value $V^{\pi(\cdot\mid s,a)}_{n-1}(s')$ starting from $s'$.
\end{lemma}

\begin{proof}
Suppose for the sake of contradiction that the continuation policy $\pi(\cdot \mid s, a)$ is \textit{not} Pareto-optimal for the state $s'$. 

This implies there exists an alternative continuation policy $\tilde{\pi}$ such that its value strictly dominates the original continuation value at $s'$:
\[
V^{\tilde{\pi}}_{n-1}(s') \succ V^{\pi(\cdot \mid s, a)}_{n-1}(s').
\]
We can construct a new global policy $\pi^*$ by ``stitching'' the original policy behavior at the first step and replacing the behavior for the history $(s, a)$ and state $s'$ onwards with $\tilde{\pi}$.


By the linearity of expectation, the total value of this new policy is:
\[
V^{\pi^*}_n(s) = R(s,a) + \gamma \E_{\hat{s}} V^{\pi^*}_{n-1}(\hat{s}).
\]
Since $\pi^*$ behaves identically to $\pi$ for all states $\hat{s} \neq s'$ (and behaves like $\tilde{\pi}$ at $s'$), the difference in value is:
\[
V^{\pi^*}_n(s) - V^{\pi}_n(s) = \gamma \cdot p(s'|s,a) \cdot \left( V^{\tilde{\pi}}_{n-1}(s') - V^{\pi(\cdot \mid s, a)}_{n-1}(s') \right).
\]
Since the term in the parentheses is strictly positive (by our assumption of dominance), it follows that $V^{\pi^*}_n(s) \succ V^{\pi}_n(s)$.

This implies that $\pi^*$ dominates $\pi$, which contradicts the initial assumption that $\pi$ was Pareto-optimal. Therefore, the continuation policy must be Pareto-optimal.
\end{proof}

\begin{theorem}[Duality of $n$-th Pareto-Optimal Estimate and $n$-step Pareto-Optimal Policy]\label{app:thm:duality}
Let $\mathcal{T}^n Q(s,a,w)$ denote the $n$-th Bellman estimate. The following relationships hold between the Pareto optimality of the Bellman estimates and the Pareto optimality of the resulting policies:

\begin{enumerate}
    \item \textbf{$n$-Step Preference Optimality $\implies$ $n$-Step Policy Optimality} \\
    Let $w \in \mathcal{W}$ be a preference such that the corresponding Bellman estimate is Pareto-optimal. That is, assume:
    \[
    \forall w' \in \mathcal{W}, \quad \mathcal{T}^n Q(s,a,w) \nprec \mathcal{T}^n Q(s,a,w').
    \]
    Then, the  policy $\pi = \pi_n^{s,a,w}$ is Pareto-optimal with respect to the $n$-step return among all policies starting with the same state $s$ and the same action $a$. Specifically, for every history-dependent deterministic policy $\pi' \in \Pi$ such that $\pi'(s) = \pi(s)=a$:
    \[
    V^{\pi}_n(s) \nprec V^{\pi'}_n(s).
    \]

    \item \textbf{$n$-Step Policy Optimality $\implies$ $n$-Step Preference Optimality} \\
    Let $\pi$ be any history-dependent deterministic policy such that $\pi(s) = a$ and its $n$-step value is Pareto-optimal among all such policies (i.e., $\forall \pi' \in \Pi$ with $\pi'(s)=a$, $V^{\pi}_n(s) \nprec V^{\pi'}_n(s)$).
    
    Define the specific preference vector $w^*$ based on the direction of the future expected return:
    \[
    w^* = \frac{V^\pi_n(s) + \gamma^n \mathcal{R} - R(s,a)}{\left\| V^\pi_n(s) + \gamma^n \mathcal{R} - R(s,a) \right\|_2}.
    \]
    Then, the Bellman estimate targeting this preference recovers the policy's value and is itself Pareto-optimal:
    \[
    \mathcal{T}^n Q(s, a, w^*) = V^\pi_n(s) + \gamma^n \mathcal{R},
    \]
    and
    \[
    \forall w' \in \mathcal{W}, \quad \mathcal{T}^n Q(s, a, w^*) \nprec \mathcal{T}^n Q(s, a, w').
    \]
\end{enumerate}
\end{theorem}

\begin{proof}
We prove both statements by induction on the horizon $n$.

\textbf{Base Case ($n=1$):}
For any $s, a, w$, $\mathcal{T}^1 Q(s,a,w) = R(s,a) + \gamma \mathcal{R}$ and $V^\pi_1(s) = R(s,a)$. Since $R(s,a)$ is constant given $a$, all estimates and policy values are identical. No value strictly dominates another, so all are Pareto-optimal. Both statements hold vacuously.

\textbf{Inductive Step:}
Assume both statements hold for horizon $n-1$. We prove them for horizon $n$.

\textbf{Part 1: Preference Optimality $\implies$ Policy Optimality}
Assume $\mathcal{T}^n Q(s,a,w)$ is Pareto-optimal. Let $\pi = \pi_n^{s,a,w}$.
Suppose for contradiction that $\pi$ is \textbf{not} Pareto-optimal. Then there exists a policy $\pi'$ (where $\pi'(s)=a$) such that $V^{\pi'}_n(s) \succ V^{\pi}_n(s)$.
Expanding the values:
\[
R(s,a) + \gamma \mathbb{E}_{s'}\left[ V^{\pi'(\cdot|s,a)}_{n-1}(s') \right] \succ R(s,a) + \gamma \mathbb{E}_{s'}\left[ V^{\pi(\cdot|s,a)}_{n-1}(s') \right].
\]
By Lemma~\ref{lemma:subpolicy_optimality}, the continuation of $\pi'$ must be Pareto-optimal for the $(n-1)$ horizon (otherwise we could improve it to get an even better policy). 
By the \textbf{Induction Hypothesis (Part 2)}, for every state $s'$, there exists a preference $w'_{s'}$ such that the Bellman update recovers the value of this continuation policy:
\[
\mathcal{T}^{n-1}Q(s', \pi'(s'\mid s,a), w'_{s'}) = V^{\pi'(\cdot|s,a)}_{n-1}(s') + \gamma^{n-1}\mathcal{R}.
\]
Now, consider the global Bellman update $\mathcal{T}^n$ calculated with a preference $w_{\text{dom}}$ that is aligned with the expected value vector $\E\limits_{s'}V_{n-1}^{\pi'(\cdot \mid s,a)}(s')$ of $\pi'$ after the first step.
By the Principle of Aligned Dominance~\ref{lem:aligned_dominance}, since the Bellman operator maximizes  scalarized expectation over all possible next actions and preferences, the resulting value must be at least as high as that of $\pi'$ in the direction of $w_{\text{dom}}$. 
Consequently, there exists a preference $w_{\text{dom}}$ such that:
\[
\mathcal{T}^n Q(s, a, w_{\text{dom}}) \succeq V^{\pi'}_n(s) + \gamma^n\mathcal{R} \succ V^{\pi}_n(s) + \gamma^n\mathcal{R}.
\]
Using the consistency property (Lemma~\ref{app:lemm:bellman_consistency}), $V^{\pi}_n(s) + \gamma^n\mathcal{R} = \mathcal{T}^n Q(s,a,w)$. Thus:
\[
\mathcal{T}^n Q(s, a, w_{\text{dom}}) \succ \mathcal{T}^n Q(s,a,w).
\]
This contradicts the assumption that $\mathcal{T}^n Q(s,a,w)$ was Pareto-optimal. 
Thus we arrive at a contradiction because of our assumption that $\pi$ is not $n$-step Pareto-optimal starting from the state $s$. 
So, $\pi$ has to be $n$-step Pareto-optimal.

\textbf{Part 2: $n$-Step Policy Optimality $\implies$ $n$-Step Preference Optimality}

Assume $\pi$ is a Pareto-optimal policy for the first $n$-steps with $\pi(s)=a$. That is, for any $\pi' \in \Pi$ such that $\pi'(s)=\pi(s)$, we have $V^{\pi}_n(s) \nprec V^{\pi'}_n(s)$.

We define the expected future return vector adjusted for the tail:
\[
U^* := V^\pi_n(s) - R(s,a) + \gamma^n \mathcal{R}.
\]
We define the target preference $w^*$ to be aligned with this return:
\[
w^* := \frac{U^*}{\|U^*\|_2}.
\]
We now analyze the Bellman update $\mathcal{T}^n Q(s,a,w^*)$. By Definition~\ref{def:greedy_selectors}, this operator computes the expectation over the optimal next-step greedy selectors:
\[
\mathcal{T}^n Q(s,a,w^*) = R(s,a) + \gamma \mathbb{E}_{s'} \left[ \mathcal{T}^{n-1}Q\Big(s', \hat{A}_{n-1}(s,a,w^*,s'), \hat{W}_{n-1}(s,a,w^*,s')\Big) \right].
\]

\textbf{Step 2.1: Feasibility (Lower Bound)} \\
By Lemma~\ref{lemma:subpolicy_optimality}, the continuation policy $\pi(\cdot|s,a)$ is Pareto-optimal for the $(n-1)$ horizon.
By the \textbf{Inductive Hypothesis (Part 2)}, for every next state $s'$, there exists a preference $w_{s'}$ such that the $(n-1)$-th estimate recovers this continuation value:
\[
\mathcal{T}^{n-1}Q(s', \pi(s'|s,a), w_{s'}) = V^{\pi(\cdot|s,a)}_{n-1}(s') + \gamma^{n-1}\mathcal{R}.
\]
The Bellman operator maximizes the scalarized objective over \textit{all} possible next-step actions and preferences. Since the specific action chosen by the policy, $a'_{s'} = \pi(s' \mid s, a)$, and the preference $w_{s'}$ constitute a valid candidate pair $(a'_{s'}, w_{s'})$ in this search space, the scalar value achieved by the optimal selectors $\hat{A}_{n-1}, \hat{W}_{n-1}$ must be at least that of the configuration used by $\pi$:
\[
O_{w^*}(\mathcal{T}^n Q(s,a,w^*) - R(s,a)) \ge O_{w^*}(U^*).
\]

\textbf{Step 2.2: Contradiction (Equality)} \\
Suppose for the sake of contradiction that the inequality is strict. That is, the Bellman operator finds a configuration of greedy selectors $\hat{A}_{n-1}$ and $\hat{W}_{n-1}$ that yields a strictly higher scalarized value:
\[
O_{w^*}\left( \gamma \mathbb{E}_{s'} \left[ \mathcal{T}^{n-1}Q(s', \hat{A}_{n-1}(\dots), \hat{W}_{n-1}(\dots)) \right] \right) > O_{w^*}(U^*).
\]
Let $V_{new}$ denote this superior expected value found by the operator. By the \textbf{Inductive Hypothesis (Part 1)}, for each next state $s'$, the term $\mathcal{T}^{n-1}Q(s', \hat{A}_{n-1}(\dots), \hat{W}_{n-1}(\dots))$ corresponds to the value of some valid Pareto-optimal policy $\pi'_{s'}$ starting at $s'$.

We can ``stitch" these policies together to form a new global policy $\pi_{stitch}$. We define $\pi_{stitch}$ such that it executes action $a$ at the root state $s$, and for any subsequent state $s'$ reached, it adopts the behavior of the corresponding sub-policy $\pi'_{s'}$. Formally, for a history $h$:
\[
\pi_{stitch}(h) = 
\begin{cases} 
a & \text{if } h = (s) \\
\pi'_{s'}(h') & \text{if } h = (s, a ) \circ h'
\end{cases}
\]
where $h'$ denotes the local history suffix starting from $s'$.

The value of this new policy is precisely the value found by the operator:
\[
V^{\pi_{stitch}}_n(s) - R(s,a) + \gamma^n \mathcal{R}= \gamma \mathbb{E}_{s'} \left[ \mathcal{T}^{n-1}Q(s', \hat{A}_{n-1}(\dots), \hat{W}_{n-1}(\dots)) \right]
\]
Substituting this back into our inequality:
\[
O_{w^*}(V^{\pi_{stitch}}_n(s) - R(s,a) + \gamma^n \mathcal{R}) > O_{w^*}(V^{\pi}_n(s) - R(s,a) + \gamma^n \mathcal{R}).
\]
By Lemma~\ref{lem:aligned_dominance} (Principle of Aligned Dominance), since $w^*$ is the unit vector strictly aligned with the direction of the RHS vector, a strictly higher scalarization score implies strict vector dominance:
\[
V^{\pi_{stitch}}_n(s) \succ V^{\pi}_n(s).
\]
This contradicts the assumption that $\pi$ is Pareto-optimal for the $n$-step discounted rewards. Therefore, strict inequality is impossible, and we must have equality:
\[
\mathcal{T}^n Q(s,a,w^*) = V^\pi_n(s) + \gamma^n \mathcal{R}.
\]

\textbf{Step 2.3: Pareto Optimality of the Estimate} \\
Finally, we prove that the estimate $\mathcal{T}^n Q(s,a,w^*)$ is itself Pareto-optimal. Suppose not; then there would exist a preference vector $w'$ such that:
\[
\mathcal{T}^n Q(s,a,w') \succ \mathcal{T}^n Q(s,a,w^*).
\]
Without loss of generality, we assume that the estimate $\mathcal{T}^n Q(s,a,w')$ is itself Pareto-optimal (otherwise, we could simply select a further dominating estimate).

We expand the term on the left-hand side using the definition of the Bellman operator:
\[
\mathcal{T}^n Q(s,a,w') = R(s,a) + \gamma \mathbb{E}_{s' \sim p(\cdot \mid s,a)}\left[ \mathcal{T}^{n-1}Q\big(s', \hat{A}_{n-1}(s,a,w',s'), \hat{W}_{n-1}(s,a,w',s')\big) \right].
\]
Let $\hat{a}_{s'} = \hat{A}_{n-1}(s,a,w',s')$ and $\hat{w}_{s'} = \hat{W}_{n-1}(s,a,w',s')$. Since $\mathcal{T}^nQ(s,a,w')$ is Pareto-optimal, each inner term $\mathcal{T}^{n-1}Q(s', \hat{a}_{s'}, \hat{w}_{s'})$ must be a Pareto-optimal estimate for the $(n-1)$ horizon (otherwise, we could replace it with a dominating term to obtain a quantity strictly more than $\T^nQ(s,a,w')$ leading to a contradiction of $\T^nQ(s,a,w')$  being Pareto-optimal).

By the \textbf{Inductive Hypothesis (Part 1)}, for each $s'$, there exists a Pareto-optimal policy $\pi'_{s'}$ starting at $s'$ such that its value matches this estimate:
\[
V^{\pi'_{s'}}_{n-1}(s') + \gamma^{n-1}\mathcal{R} = \mathcal{T}^{n-1}Q(s', \hat{a}_{s'}, \hat{w}_{s'}).
\]
We can now construct a new global policy $\pi_{new}$ that takes action $a$ at state $s$, and for any subsequent history $h = ((s, a) , h')$, follows the policy $\pi'_{s'}(h')$. The value of this stitched policy is:
\[
V^{\pi_{new}}_n(s) + \gamma^n \mathcal{R} = R(s,a) + \gamma \mathbb{E}_{s'} \left[ V^{\pi'_{s'}}_{n-1}(s') + \gamma^{n-1}\mathcal{R} \right] = \mathcal{T}^n Q(s,a,w').
\]
Combining this with our initial dominance assumption:
\[
V^{\pi_{new}}_n(s) + \gamma^n \mathcal{R} \succ \mathcal{T}^n Q(s,a,w^*) = V^\pi_n(s) + \gamma^n \mathcal{R}.
\]
Simplifying, we get $V^{\pi_{new}}_n(s) \succ V^\pi_n(s)$. This contradicts the assumption that $\pi$ is Pareto-optimal. Thus, no such dominating preference $w'$ can exist.
\end{proof}

In the next theorem, we establish the transformation in the preference required to obtain a one-to-one correspondence between Pareto-optimal policies and Bellman estimates.

\begin{theorem}[Bellman-Policy Pareto Equivalence and Scalarization Alignment]
Combining the Bellman Consistency (Lemma~\ref{app:lemm:bellman_consistency}) and the Pareto Duality (Theorem~\ref{app:thm:duality}), we establish the following equivalence:

\begin{enumerate}
    \item \textbf{Estimates generate optimal policies aligned with future returns:} \\
    If the Bellman estimate $\mathcal{T}^n Q(s,a,w)$ is Pareto-optimal (i.e., $\forall w' \in \mathcal{W}, \mathcal{T}^nQ(s,a,w) \nprec \mathcal{T}^nQ(s,a,w')$), then the resulting policy $\pi = \pi_n^{s,a,w}$ is Pareto-optimal for the $n$-step return among all policies starting with action $a$ (i.e., $\forall \pi' \in \Pi \text{ s.t. } \pi'(s)=a, V^\pi_n(s) \nprec V^{\pi'}_n(s)$). 
    
    Furthermore, this policy maximizes the scalarized expected return for the specific preference direction $w_{\text{eff}}$ aligned with the estimated value adjusted for the tail:
    \[
    w_{\text{eff}} = \frac{\mathcal{T}^n Q(s,a,w) - \gamma^n \mathcal{R}}{\left\| \mathcal{T}^n Q(s,a,w) - \gamma^n \mathcal{R} \right\|_2}.
    \]
    That is, $\forall \pi' \in \Pi \text{ s.t. } \pi'(s)=a, \quad O_{w_{\text{eff}}}(V^\pi_n(s)) \geq O_{w_{\text{eff}}}(V^{\pi'}_n(s))$.

    \item \textbf{Optimal policies correspond to specific Bellman estimates:} \\
    Conversely, let $\pi$ be a history-dependent policy that takes action $a$ at state $s$ and is Pareto-optimal for the $n$-step horizon among all such policies (i.e., $\forall \pi' \in \Pi \text{ s.t. } \pi'(s)=\pi(s)=a, V^{\pi}_n(s) \nprec V^{\pi'}_n(s)$).

    Then, its value is exactly recovered by the Bellman operator targeting the preference $w^*$ aligned with the policy's actual future return:
    \[
    w^* = \frac{V^\pi_n(s) - R(s,a) + \gamma^n \mathcal{R}}{\left\| V^\pi_n(s) - R(s,a) + \gamma^n \mathcal{R} \right\|_2}.
    \]
    Specifically, the Bellman estimate recovers the policy value:
    \[
    \mathcal{T}^n Q(s, a, w^*) = V^\pi_n(s) + \gamma^n \mathcal{R},
    \]
    and this estimate is itself Pareto-optimal (i.e., $\forall w' \in \mathcal{W}, \mathcal{T}^n Q(s, a, w^*) \nprec \mathcal{T}^n Q(s, a, w')$).

\end{enumerate}
\end{theorem}
\begin{proof}[Proof sketch]
    Due to Theorem~\ref{app:thm:duality} and Principle of Aligned Dominance Lemma~\ref{lem:aligned_dominance}.
\end{proof}

Next we give the exact algorithm for greedy policy based on the $n^{th}$-Bellman estimate.

\begin{algorithm}[tb]
  \caption{Pareto-Optimal Policy Execution}
  \label{alg:pareto_execution}
  \begin{algorithmic}[1]
    \renewcommand{\baselinestretch}{1.2}
    \selectfont

    \REQUIRE 
    \begin{minipage}[t]{0.85\linewidth}
        { \renewcommand{\baselinestretch}{1.0} \selectfont
        \begin{itemize}
            \setlength{\itemsep}{0pt}
            \setlength{\parskip}{0pt}
            \setlength{\leftmargin}{1em}
            \item Value Function $\mathcal{T}^n Q$
            \item Transition kernel $p(\cdot \mid s, a)$ and Reward function $R(s, a)$
            \item Initial state $s_0$ and Initial Preference $w_{\text{init}}$
        \end{itemize}
        }
    \end{minipage}
    
    \vspace{0.2cm}
    \STATE \textbf{Initialization:}
    \STATE Identify the set of optimal candidates at $s_0$:
    \STATE \quad $S_{\text{opt}} = \operatorname*{argmax}_{a', w'} O_{w_{\text{init}}}\big( \mathcal{T}^n Q(s_0, a', w') \big)$
    \STATE Select candidate with maximum norm:
    \STATE \quad $(\hat{a}_0, \hat{w}_1) = \operatorname*{argmax}_{(a', w') \in S_{\text{opt}}} \big\| \mathcal{T}^n Q(s_0, a', w') \big\|_2$
    
    \STATE Initialize current action $a_0 \leftarrow \hat{a}_0$ and target preference $w_{\text{target}} \leftarrow \hat{w}_1$.
    \STATE Set time step $t \leftarrow 0$.

    \vspace{0.2cm}
    \WHILE{not termination}
        \vspace{0.1cm}
        \STATE Execute $a_t$, observe next state $s_{t+1}$ and reward $r_t = R(s_t, a_t)$.

        \STATE \textit{// Step 1: Preference Alignment}
        \STATE Calculate preference vector aligned with expected tail rewards:
        \STATE \quad $\mathbf{v}_{\text{tail}} = \mathcal{T}^n Q(s_t, a_t, w_{\text{target}}) - R(s_t, a_t) - \gamma^n \mathcal{R}$
        \STATE \quad $\hat{w}_{t+1} = \frac{\mathbf{v}_{\text{tail}}}{\| \mathbf{v}_{\text{tail}} \|_2}$

        \vspace{0.1cm}
        \STATE \textit{// Step 2: Chebyshev Optimization (Lookahead)}
        \STATE Compute policy maps maximizing scalarization w.r.t.\ $\hat{w}_{t+1}$:
        \STATE \quad $(\hat{A}, \hat{W}) = \operatorname*{argmax}_{\substack{\phi_A: S \to A \\ \phi_W: S \to \mathcal{W}}} O_{\hat{w}_{t+1}} \left( \mathbb{E}_{s' \sim p(\cdot \mid s_t, a_t)} \big[ \mathcal{T}^n Q(s', \phi_A(s'), \phi_W(s')) \big] -\gamma^n\mathcal{R} \right)$
        
        \vspace{0.1cm}
        \STATE \textit{// Step 3: Selection and Update}
        \STATE Determine action and preference for the observed state $s_{t+1}$:
        \STATE \quad $a_{t+1} = \hat{A}(s_{t+1})$
        \STATE \quad $w_{\text{next}} = \hat{W}(s_{t+1})$
        
        \STATE Update history trackers:
        \STATE \quad $w_{\text{target}} \leftarrow w_{\text{next}}$
        \STATE \quad $t \leftarrow t + 1$
    \ENDWHILE
  \end{algorithmic}
\end{algorithm}

\noindent \textbf{Description:} The execution procedure is detailed in Algorithm~\ref{alg:pareto_execution}. The process begins (Lines~2--5) by identifying the optimal action and next-preference pair $(\hat{a}_0, \hat{w}_1)$ for the initial state $s_0$. 

The core of the execution loop relies on dynamic preference alignment. By Lemma~\ref{app:lemm:bellman_consistency} the current estimate $\mathcal{T}^n Q(s_t, a_t, w_{\text{target}})$ corresponds to the value of an $n$-step policy ($\pi^{s,a,w}_n$), specifically $\mathcal{T}^n Q - R - \gamma^n \mathcal{R}$. In Line~12 and 13, we explicitly calculate the preference direction $\hat{w}_{t+1}$ that aligns with this future value component.

In Line~15, the algorithm performs a search to find the action and preference maps $(\hat{A}, \hat{W})$ that greedily maximize the scalarized expected return with respect to the aligned preference $\hat{w}_{t+1}$. 
By the \textit{Principle of Aligned Dominance} (Lemma~\ref{lem:aligned_dominance}), maximizing this specific scalarization guarantees that the resulting policy choice dominates the future component of the previous estimate. 
Finally, in Lines~19--20, the agent instantiates this policy for the actually observed state $s_{t+1}$ to determine the next action $a_{t+1}$ and target preference $w_{\text{next}}$.

It is important to note that while the recursive calculation of the Bellman operator typically utilizes the $(n-1)$-th estimate, our execution policy consistently utilizes the full $n$-th estimate $\mathcal{T}^n Q$ for the lookahead optimization at every step, ensuring the highest available lookahead in action selection.

\begin{theorem}[Recursive Trajectory Error Bound]
\label{app:thm:recursive_trajectory_error}
Let $\{(s_t, a_t, w_{t+1})\}_{t \ge 0}$ be the trajectory generated by the execution policy, where $\hat{w}_{t+1}$ is the preference target aligned for the next step. Let $V^\pi(s)$ denote the true value of the executed policy starting from state $s$. 

The following recursive inequality holds:
\[
\mathcal{T}^n Q(s_t, a_t, w_{t+1}) - V^\pi(s_t) \preceq \gamma \mathbb{E}_{s_{t+1} \sim p(\cdot|s_t, a_t)} \left[ \mathcal{T}^n Q(s_{t+1}, a_{t+1}, w_{t+2}) - V^\pi(s_{t+1}) \right] + \gamma^n(1-\gamma)\R
\]
\end{theorem}

\begin{proof}
Let $\Delta_t = \mathcal{T}^n Q(s_t, a_t, w_{t+1}) - V^\pi(s_t)$ denote the vector error at time step $t$. We expand both terms using their respective Bellman definitions.

First, we expand the true policy value $V^\pi(s_t)$ using the standard Bellman equation:
\begin{equation}
    V^\pi(s_t) = R(s_t, a_t) + \gamma \mathbb{E}_{s_{t+1} \sim p(\cdot|s_t, a_t)}\left[ V^{\pi(\cdot\mid s_t,a_t)}(s_{t+1}) \right].
    \label{eq:true_val_expansion}
\end{equation}

Next, we expand the $n$-step estimate $\mathcal{T}^n Q(s_t, a_t, w_{t+1})$. From Lemma~\ref{app:lemm:bellman_consistency} we obtain that this estimate is composed of the immediate reward plus the discounted value of the \textit{tail policy} plus a residual value:
\begin{equation}
    \mathcal{T}^n Q(s_t, a_t, w_{t+1}) = R(s_t, a_t) + \gamma \mathbb{E}_{s_{t+1} \sim p(\cdot|s_t, a_t)}\left[ V^{\text{tail}}_{n-1}(s_{t+1}) + \gamma^{n-1}\mathcal{R} \right].
    \label{eq:est_val_expansion}
\end{equation}

Subtracting \eqref{eq:true_val_expansion} from \eqref{eq:est_val_expansion}, the immediate reward $R(s_t, a_t)$ cancels out:
\[
\Delta_t = \gamma \mathbb{E}_{s_{t+1}} \left[ \left( V^{\text{tail}}_{n-1}(s_{t+1}) + \gamma^{n-1}\mathcal{R} \right) - V^{\pi(\cdot\mid s_t,a_t)}(s_{t+1}) \right].
\]

At step $t+1$, the agent re-computes the estimate $\mathcal{T}^n Q(s_{t+1}, a_{t+1}, w_{t+2})$ using a full $n$-step horizon and an updated preference $\hat{w}_{t+1}$ explicitly aligned with $\E_{s_{t+1}}V^{\text{tail}}_n(s_{t+1})$. This fresh estimate is guaranteed to weakly dominate the residual value of the previous plan (the stale tail) by the Principle of Aligned Dominance (Lemma~\ref{lem:aligned_dominance}):
\[
V^{\text{tail}}_{n-1}(s_{t+1})  \;\preceq\; \mathcal{T}^n Q(s_{t+1}, a_{t+1}, w_{t+2})-\gamma^n\R
\]

Substituting this inequality back into the expression for $\Delta_t$, we obtain:
\begin{align*}
\Delta_t &\preceq \gamma \mathbb{E}_{s_{t+1}} \left[ \mathcal{T}^n Q(s_{t+1}, a_{t+1}, w_{t+1}) - V^{\pi(\cdot\mid s_t,a_t)}(s_{t+1}) \right] + \gamma^n(1-\gamma)\R\\
&\preceq \gamma \mathbb{E}_{s_{t+1}} \left[ \Delta_{t+1} \right] + \gamma^n(1-\gamma)\R
\end{align*}

Thus, the error at time $t$ is bounded by the discounted expected error at time $t+1$.
\end{proof}

\begin{theorem}[Approximate Pareto Coverage]\label{app:thm:app_pareto_coverage}
    Let $\mathcal{C}^n_{PO} = \{ \pi_w \mid w \in \mathcal{W} \}$ denote the set of policies generated by all preferences using the $n$-th Bellman update by Algorithm~\ref{alg:high_level_pareto_policy}. This set satisfies the properties of approximate Pareto coverage. 
    

    \begin{enumerate}[label=\thetheorem.\arabic*]
        \item \textbf{Approximate Coverage:} 
        Every Pareto-optimal policy is approximated by some policy in $\mathcal{C}^n_{PO}$ such that the \textbf{suboptimality gap} is bounded. Formally:
        \[
        \begin{aligned}
            &\forall \pi \in \Pi, \exists w \in \mathcal{W} \text{ such that } \\
            & V^{\pi}(s_0) \preceq V^{\pi_w}(s_0) + \gamma^n \mathcal{R}
        \end{aligned}
        \]

        \item \textbf{Approximate Parsimony:} 
        Every policy in $\mathcal{C}^n_{PO}$ satisfies \textbf{$\bm\gamma^{\bm n}\bm\R$-Pareto optimality}; that is, no other policy can strictly dominate it by a margin larger than the approximation bound. Formally:
        \[
        \begin{aligned}
            &\forall w \in \mathcal{W}, \forall \pi \in \Pi \quad \text{ we get} \\
            &V^{\pi}(s_0) \nsucc V^{\pi_w}(s_0) + \gamma^n \mathcal{R}
        \end{aligned}
        \]
    
    \end{enumerate}

\end{theorem}
\begin{proof}
    Using Theorem~\ref{app:thm:recursive_trajectory_error}, we obtain $\Delta_t\preceq \gamma \E_{s_{t+1}}[\Delta_{t+1}] + \gamma^n(1-\gamma)\R$, therefore taking a max for each component allows us to prove $\Delta_{max}\preceq \gamma^n\R$.
    Therefore starting with any $n^{th}$ estimate $\T^nQ(s,a,w)$ the procedure obtains a value $V^\pi(s)$, such that $\T^nQ(s,a,w)\preceq V^\pi(s)+\gamma^n\R$.
    \begin{enumerate}
        \item \textbf{Approximate Coverage:} Since we prove in Theorem~\ref{prop:opt_coverage} that every Pareto-optimal policy $\pi^*$ is dominated by an estimate $V^{\pi^*}(s)\preceq \T^nQ(s,a,w')\preceq V^{\pi_{w'}}(s)+\gamma^n\R$.
        \item \textbf{Approximate Parsimony:} Since there is no policy that dominates an estimate $\T^nQ(s,a,w)$ we obtain that there is no policy that dominates $V^{\pi_w}(s_0)+\gamma^n\R$.
    \end{enumerate}
\end{proof}



\end{document}